%% file: 0_main.tex
\pdfoutput=1
\documentclass{article}
\PassOptionsToPackage{numbers}{natbib}
\usepackage[preprint]{neurips_2023}

\usepackage[utf8]{inputenc} % allow utf-8 input
\usepackage[T1]{fontenc}    % use 8-bit T1 fonts
\usepackage{hyperref}
\usepackage{url}            % simple URL typesetting
\usepackage{booktabs}       % professional-quality tables
\usepackage{amsfonts}       % blackboard math symbols
\usepackage{nicefrac}       % compact symbols for 1/2, etc.
\usepackage{microtype}      % microtypography
\usepackage{xcolor}         % colors

\usepackage{amsmath}
\usepackage[capitalise]{cleveref}
\usepackage{bm}
\usepackage{algorithm}
\usepackage{algorithmic}
\usepackage{amsthm}
\usepackage{subcaption}
\usepackage{graphicx}
\usepackage{threeparttable, array, float}

\newtheorem{theorem}{\bf Theorem}
\newtheorem{lemma}{\bf Lemma}

\newcommand{\compilehidecomments}{false}%HIDE comments
\ifthenelse{ \equal{\compilehidecomments}{true} }{%
	\newcommand{\jinhang}[1]{}
	\newcommand{\mo}[1]{}
    \newcommand{\rev}[1]{}
    \newcommand{\adam}[1]{}
}{
\newcommand{\jinhang}[1]{{\color{orange} [\text{Jinhang:} #1]}}
\newcommand{\mo}[1]{{\color{blue} [\text{MH:} #1]}}
\newcommand{\rev}[1]{{\color{red}#1}}
\newcommand{\adam}[1]{{\color{blue} [\text{AW:} #1]}}
}

\title{Adversarial Attacks on Online Learning to Rank with Click Feedback}
\author{
Jinhang Zuo$^{1\,2}$ \quad Zhiyao Zhang$^{3}$ \quad Zhiyong Wang$^{4}$ \quad \textbf{Shuai Li}$^{5}$ \\ 
\textbf{Mohammad Hajiesmaili}$^{1}$ \quad \textbf{Adam Wierman}$^{2}$\\
$^{1}$University of Massachusetts Amherst\,
$^{2}$California Institute of Technology\\
$^{3}$Southeast University\,
$^{4}$The Chinese University of Hong Kong\,
$^{5}$Shanghai Jiao Tong University 
}

\begin{document}

\maketitle

\begin{abstract}
Online learning to rank (OLTR) is a sequential decision-making problem where a learning agent selects an ordered list of items and receives feedback through user clicks. Although potential attacks against OLTR algorithms may cause serious losses in real-world applications, little is known about adversarial attacks on OLTR. This paper studies attack strategies against multiple variants of OLTR. Our first result provides an attack strategy against the UCB algorithm on classical stochastic bandits with binary feedback, which solves the key issues caused by bounded and discrete feedback that previous works can not handle. Building on this result, we design attack algorithms against UCB-based OLTR algorithms in position-based and cascade models. Finally, we propose a general attack strategy against any algorithm under the general click model. Each attack algorithm manipulates the learning agent into choosing the target attack item $T-o(T)$ times, incurring a cumulative cost of $o(T)$. Experiments on synthetic and real data further validate the effectiveness of our proposed attack algorithms.
\end{abstract}

\input{1_intro.tex}
\input{2_model.tex}

\input{3_binary}

\input{4_oltr}
\input{5_experiment}

\section{Concluding Remarks}
% \adam{add discussion of Limitations}
This paper initiates the first study of adversarial attacks on OLTR with click feedback. We propose attack strategies that can successfully mislead different OLTR algorithms under different click models. One limitation of our work is that it cannot be applied to feedback models without clicks such as top-$K$ feedback~\cite{chaudhuri2017online}. Though we provide theoretical results for all these strategies, as discussed in \Cref{sec:ucb_analysis}, the finite-time cost results might be further improved by a more fine-grained analysis on $h_a(t)$.
Also, there is no known lower bound of the cumulative attack cost even for the stochastic bandit setting in literature, so it is unclear whether our attack strategies are (asymptotically) order-optimal. It also opens up a number of future directions. One is to design attack strategies for other OLTR algorithms and feedback models. Lastly, our study on the vulnerability of existing OLTR algorithms inspires robust algorithm design against adversarial attacks  OLTR.

\bibliographystyle{unsrt}
\bibliography{0_main.bib}

\input{appendix}

\end{document}

%% file: 1_intro.tex
\section{Introduction}\label{sec:intro}
% Background \& Motivation
% Challenge
% \begin{itemize}
%     \item Attack on (binary) click feedback: hard to take control on the attack value
%     \item Attack on PBM model: position-based estimator
%     \item Attack on cascade model: partial feedback
% \end{itemize}

% Results
% \begin{itemize}
%     \item Attack on (binary) click feedback
%     \item Attack on PBM-UCB
%     \item Attack on CascadeUCB
%     \item General attack on click feedback
%     \item Experiments on synthetic and real data
% \end{itemize}

Online learning to rank (OLTR) has been extensively studied~\cite{zoghi2017online,li2019online,lattimore2018toprank,lagree2016multiple,kveton2015cascading} as a sequential decision-making problem where, in each round, a learning agent presents a list of items to users and receives implicit feedback from user interactions. One of the most common feedback considered in literature is in the form of user clicks~\cite{zoghi2017online,lagree2016multiple,kveton2015cascading}. OLTR with such click feedback, can lead to major improvements over traditional supervised learning to rank methods~\cite{cao2007learning,liu2009learning,cakir2019deep}. However, there exists a security caveat that user-generated click feedback might be generated by malicious users with the goal of manipulating the learning agent. Understanding the vulnerability of OLTR under adversarial attacks plays an essential role in developing effective defense mechanisms for trustworthy OLTR systems. 

There has been a surge of interest in adversarial attacks on multi-armed bandits~\cite{jun2018adversarial,liu2019data,xu2021observation}. For example, \cite{jun2018adversarial} shows that, for stochastic bandits, it is possible to manipulate the bandit algorithm into pulling a target arm very often with sublinear cumulative cost.
Though OLTR generally follows the bandit formulation, it differs from stochastic bandits in the action space and feedback model. 
More specifically, OLTR chooses a list of $K$ out of $L$ arms, instead of a single arm, to play in each round; the realized rewards of the chosen arms are usually censored by a click model, e.g., position-based~\cite{lagree2016multiple} or cascade model~\cite{kveton2015cascading}. Thus, it is nontrivial to design efficient adversarial attacks on the censored feedback of the chosen arms for OLTR. 

Moreover, previous works~\cite{jun2018adversarial,liu2019data} can only handle unbounded and continuous reward feedback, while the (binary) click feedback of OLTR is bounded and discrete. Such binary feedback brings new challenges to the attack design against OLTR algorithms. 
Since the post-attack feedback must also be binary, it introduces a new problem of deciding whether to attack when the output attack value from previous algorithms is between $0$ and $1$: a simple rounding up might be costly, while skipping the attack may lead to undesired non-target arm pulls.
Furthermore, the attack value computed by previous attack algorithms can be larger than $1$, which is higher than the maximum attack value in the click model. In other words, in the bounded and discrete model, it is impossible to always find a feasible attack value to ensure the required conditions for their theoretical guarantees.

\begin{table}[t]
    \centering
	\caption{Summary of the settings and proposed attack strategies$^\dagger$}
	\label{tab:algorithms}
\resizebox{0.97\columnwidth}{!}{
\begin{threeparttable}
	\begin{tabular}{cccc}
		\toprule
		\textbf{Setting}&\textbf{Attack against}& $N_L(T)$ & $\lim C(T) / \log T$\\
		\midrule
		$L$-armed bandits & \textit{UCB}~\cite{bubeck2012regret} & $T - O\left((L-1)\left(\frac{1}{\Delta_0^2} \log T\right)\right)$ & $O\left(\sum_{a<L}\frac{\Delta_a + \Delta_0}{\Delta_0^2} \right)$ \\
	Position-based model &\textit{PBM-UCB}~\cite{lagree2016multiple} & $T - O\left((L-K)\left(\frac{1+\epsilon}{\kappa^2_{K}\Delta_0^2} \log T\right)\right)$ & $O\left(\sum_{a<L}\frac{(1+\epsilon)(\Delta_a + \Delta_0)}{\kappa^2_{M}\Delta_0^2} \right)$ \\
		Cascade model &\textit{CascadeUCB}~\cite{kveton2015cascading} & $T - O\left((L-K)\left(\frac{1}{p^{*}\Delta_0^2} \log T\right)\right)$ & $O\left(\sum_{a<L}\frac{\Delta_a + \Delta_0}{\Delta_0^2} \right)$  \\
      General model & \textit{Arbitrary} & $T - O(\log T)$ & $O\left(\sum_{a<L}(\Delta_a + 4\beta(1)) \right)$  \\
		\bottomrule
	\end{tabular}
   \begin{tablenotes}[para, online,flushleft]
	\footnotesize%\smallskip
	\item[]\hspace*{-\fontdimen2\font}$^\dagger$ $\Delta_0$: parameter of the attack algorithm; $\Delta_a$: mean gap between arm $a$ and $L$; $\beta$: a decreasing function in \Cref{sec:attack_UCB} \end{tablenotes}
 \end{threeparttable}
}
\end{table}
\textbf{Contributions.}
In this paper, we propose the first study of adversarial attacks on OLTR with click feedback, aiming to overcome the new issues raised by the OLTR click model as well as the binary feedback. \Cref{tab:algorithms} summarizes our proposed attack algorithms with their theoretical guarantees. Since the click feedback itself complicates the attack design, we first consider adversarial attacks on stochastic $L$-armed bandits with Bernoulli rewards where the feedback from the chosen arm is always binary. 
We propose an attack algorithm that can mislead the well-known UCB algorhtm~\cite{bubeck2012regret} to pull the target arm $L$ for $N_L(T)$ times in $T$ rounds, with a cumulative cost $C(T)$ in the order of $\log T$ asymptotically. Based on it, we study the two most popular click models of OLTR, position-based model~\cite{lagree2016multiple} and cascade model~\cite{kveton2015cascading}, and propose attack strategies against UCB-type OLTR algorithms. Their cumulative costs depend on $\kappa_K$ and $p^*$, which are instance-dependent parameters of the position-based and cascade models, respectively.
Lastly, we introduce the threat model for OLTR with a general click model and design an attack algorithm that can misguide arbitrary OLTR algorithms using click feedback.
Our technical contributions are summarized as follows.
% \mo{the following items are mainly duplicate with the paragraph above with a few exceptions. what about removing it and then break the paragraph above into two paragraphs with emphasis on the statement of the result and algorithms in the 1st one, and 2nd one on the technical challenges and how we get those results?}
\begin{enumerate}
    \item We propose the new idea of \textit{conservative estimation} of the target arm for attacking UCB in stochastic bandits with binary feedback, which resolves the issues caused by the bounded and discrete requirements of the post-attack feedback. It is also the backbone of other attack strategies in more complicated OLTR settings.  
    \item The PBM-UCB algorithm uses \textit{position bias-corrected counters} other than simple click counters to compute UCB indices; we provide the theoretical analysis of our attack algorithm against PBM-UCB by carefully treating these unusual counters.
    \item The \textit{partial feedback} of OLTR with the cascade click model brings a new challenge to the attack design. We provide a new \textit{regret-based analysis} of our attack algorithm against CascadeUCB without suffering from the partial feedback issue.
    \item We devise a general attack strategy based on a new \textit{probabilistic attack design}. It can successfully attack arbitrary OLTR algorithms without knowing the details of the algorithm. 
\end{enumerate}
We also conduct experiments on both synthetic and real-world data to evaluate our proposed attack strategies. Experimental results show that they can effectively attack the corresponding OLTR algorithms with less costs comparing to other baselines. Due to space constraints, proofs and empirical results are included in the appendix.

\textbf{Related Work}.
Online learning to rank with different feedback models has attracted much attention in recent years. Though there exist other types of feedback models such as top-$K$ feedback~\cite{chaudhuri2017online}, click feedback has been widely used in literature. \cite{lagree2016multiple,komiyama2017position,ermis2020learning} consider the position-based model (PBM), where each position of the item list has an examination probability known or unknown by the learning agent. The cascade model in \cite{kveton2015cascading,kveton2015combinatorial,zong2016cascading,li2016contextual,zong2016cascading,vial2022minimax} considers the setting where the user would check the recommended items sequentially and stop at the first clicked one; all items after the clicked item will not be examined. The dependent click model (DCM) is a generalization of the cascade model where the user may click on multiple items~\cite{katariya2016dcm,liu2018contextual}. There are also general click models~\cite{zoghi2017online,li2019online,lattimore2018toprank} that can cover some of the previous click models.
In this paper, we mainly focus on attack design on PBM and cascade models; both of them adopt bounded and discrete click feedback thus require a new design other than previous works like ~\cite{jun2018adversarial} that can only handle unbounded and continuous feedback.

Adversarial attacks on different types of multi-armed bandit problems have been studied recently~\cite{jun2018adversarial,liu2019data,garcelon2020adversarial,ma2023adversarial,wang2022linear}. \cite{jun2018adversarial} proposes the first study of adversarial attacks on the classical stochastic bandit problem. It designs effective attack strategies against the $\epsilon$-Greedy and UCB algorithms. \cite{liu2019data} extends it to a more general setting where the algorithm of the learning agent can be unknown.
\cite{garcelon2020adversarial} studies adversarial attacks on linear contextual bandits where the adversarial modifications can be added to either rewards or contexts. To the best of our knowledge, we are the first to study adversarial attacks on OLTR where, in addition to the challenges raised by click feedback, the combinatorial action space and the censored feedback of OLTR make it non-trivial to design efficient attack strategies.

%% file: 2_model.tex
\section{Preliminaries}\label{sec:pre}
In this section, we briefly discuss the three problem settings we consider in this paper.
% Learning algorithms for different settings will be introduced in later sections.

\textbf{Stochastic $L$-armed bandit}. We consider an arm set $[L] = \{1,2,\cdots,L\}$, with $\mu_i$ as the expected reward 
 of arm $i$. Without loss of generality, we assume $\mu_1 \ge \mu_2 \ge \cdots \ge \mu_L$. In round $t$, the player chooses an arm $a_t \in [L]$ to play and receives a reward $r^0_{t}$ as feedback. In the click feedback setting, the realized reward $r^0_t\in \{0,1\}$ of arm $a_t$ is sampled from a Bernoulli distribution with expectation $\mu_{a_t}$. The player aims to find an optimal policy to maximize the long-term cumulative reward.

\textbf{OLTR with position-based model \cite{lagree2016multiple}}. This setting also considers the item (arm) set $[L]$ where $\mu_i$ represents the click probability of item $i$. However, in OLTR with the position-based model (PBM), in each round $t$, the player chooses an ordered list of $K$ items, $\bm{a}_t=(a_{1,t},\cdots, a_{K,t})$, with known examination probability $\kappa_k$ for the $k$-th position in the list (assuming $\kappa_1 \ge \cdots \ge \kappa_K$). The play then observes the click feedback of the chosen list from the user, denoted as $\bm{r}^0_t = (r^0_{1,t},\cdots,r^0_{K,t})\in \{0,1\}^{K}$, where $r^0_{i,t}$ is sampled from a Bernoulli distribution with expectation $\kappa_{i} \mu_{a_{i,t}}$. The reward obtained by the player is the sum of the clicks, i.e., $\sum_{k=1}^{K} r^0_{k,t}$. The goal of the player is to find an optimal policy that can maximize the long-term cumulative user clicks.

\textbf{OLTR with cascade model \cite{kveton2015cascading}}. 
Here, we consider the same item set $[L]$ as that in the PBM model. For OLTR with cascade model, in each round $t$, the player chooses an ordered list of $K$ items, ${\bm{a}_t=(a_{1,t},\cdots, a_{K,t})}$, for the user. The user then checks the list from $a_{1,t}$ to $a_{K,t}$, with probability $\mu_{a_{k,t}}$ to click the $k$-th item. She immediately stops at the first clicked item, and returns the click result back to the player. We denote the position of the first clicked item as $\textbf{C}_t$ ($\textbf{C}_t = \infty$ if no item was clicked). The click feedback of the player is $\bm{r}^0_t = (r^0_{1,t},\cdots,r^0_{K,t})\in \{0,1\}^{K}$, where only $r^0_{\textbf{C}_t,t} = 1$ and $r^0_{k,t} = 0$ for all $k\neq \textbf{C}_t$. The reward obtained by the player is again the sum of user clicks $\sum_{k=1}^{K} r^0_{k,t}$, but in the cascade model, it is at most $1$. The goal of the player is also to find an optimal policy that can maximize the long-term cumulative user clicks.

%% file: 3_binary.tex
\section{Attacks on Stochastic L-armed Bandits with Binary Feedback}\label{sec:binary}
As mentioned in the introduction, one main challenge of attacking OLTR algorithms comes from the binary click feedback: such binary feedback limits the possible actions of the attacker since they need to ensure the post-attack reward feedback is still valid (binary). This is a common issue for all OLTR with click feedback. Hence, in this section, we focus on adversarial attacks on the $L$-armed bandit problem with binary feedback. We propose an attack algorithm against the UCB algorithm, which is the backbone of the attack strategies for more complicated OLTR settings.
\subsection{Threat Model}\label{sec:binary_model}
% \begin{enumerate}
%     \item Bandit algorithm chooses arm $a_t$
%     \item User generate pre-attack feedback $r^0_t \in \{0,1\}$ 
%     \item Attacker observes $a_t, r^0_t$, and decides post-attack feedback $r_t\in \{0,1\}$
%     \item Bandit algorithm receives $r_t$ and then chooses the next arm $a_{t+1}$
% \end{enumerate}
We first introduce the threat model for the $L$-armed bandit problem with binary feedback. In each round $t$, the player chooses an arm $a_t \in [L]$ to play. The environment generates the pre-attack reward feedback $r^0_t \in \{0,1\}$ based on a Bernoulli distribution with mean $\mu_{a_t}$. The attacker then observes $a_t, r^0_t$, and decides the post-attack feedback $r_t\in \{0,1\}$. The player only receives $r_t$ as the feedback and  uses that to decide the next arm to pull, $a_{t+1}$, for round $t+1$.
Without loss of generality, we assume arm $L$ is a sub-optimal \textit{target} arm.
The attacker's goal is to misguide the player to pull the target arm $L$ very often while using small attack costs.
Let $N_i(t)$ denote the number of pulls of arm $i$ up to round $t$. We say the attack is successful after $T$ rounds if $N_L(T) = T - o(T)$ in expectation or with high probability, while the cumulative attack cost $C(T) = \sum_{t=1}^{T}|r_t - r^0_t| = o(T)$.

% \mo{the notations and presentation of this section is too similar to NeurIPS'18 paper. I am changing slightly but we may need to use different notations as well.} 
\subsection{Attack Algorithm against UCB}\label{sec:attack_UCB}
For a better illustration of the attack strategy, we first define the following auxiliary notations. Let $\tau_{i}(t):=\{s:s\le t, a_s = i\}$ be the set of rounds up to $t$ when arm $i$ is played. We denote the pre-attack average reward of arm $i$ up to round $t$ by $\hat{\mu}^0_{i}(t) := N_i(t)^{-1} \sum_{s\in \tau_i(t)} r^0_{s}$. Last, let 
$\hat{\mu}_{i}(t) := N_i(t)^{-1} \sum_{s\in \tau_i(t)} r_{s}$ be the post-attack average reward of arm $i$ up to round $t$.

% \mo{introduce $\alpha$ and $\psi$}
As in \cite{jun2018adversarial}, we consider attacks against the $(\alpha,\psi)$-UCB algorithm from \cite{bubeck2012regret}, where $\alpha=4.5$ and $\psi: \lambda \mapsto \lambda^2/8$ since Bernoulli random variables are $1/4$-sub-Gaussian.
% The arm-selection rule of the player is:
% \begin{equation*}
%     a_{t} = \begin{cases}
%         t, &\text{if $t < L$}\\
%         \text{argmax}_{i} \left\{\hat{\mu}_{i}(t-1) + \frac{3}{2} \sqrt{\frac{\log t}{N_{i}(t-1)}}\right\}, &\text{otherwise.}\\
%     \end{cases}
% \end{equation*}
 The original attack algorithm in \cite{jun2018adversarial} calculates an attack value $\alpha_t$ for round $t$ such that 
\begin{equation}\label{eq:old_alpha}
    \hat{\mu}_{a_t}(t) \le \hat{\mu}_L(t-1) - 2\beta(N_L(t-1)) - \Delta_0,
\end{equation}
where $\beta(N):= \sqrt{\frac{1}{2N}\log \frac{\pi^2LN^2}{3\delta}}$; $\Delta_0 > 0$ and $\delta >0$ are the parameters of the attack algorithm. The attacker then gives the post-attack reward $r_t = r^0_t - \alpha_t$ back to the player.

However, this attack design only works when the reward space is unbounded and continuous, while in the threat model with binary feedback, the Bernoulli rewards are \textit{discrete} and \textit{bounded}. There are two issues raised by the discrete and bounded requirements. First, the calculated attack value $\alpha_t$ is a real number, which may not make the post-attack reward feedback $r_t$ to be valid (binary). Second, in order to ensure that \cref{eq:old_alpha} is true, the calculated $\alpha_t$ might be larger than 1, which is higher than the maximum attack value in our threat model. In other words, it is impossible to let \cref{eq:old_alpha} hold for all rounds, while such a guarantee was essential for the theoretical analysis in \cite{jun2018adversarial}.
\begin{algorithm}[t]
 \caption{Attacks against the UCB algorithm on stochastic bandits with binary feedback}\label{alg:UCB}
 \begin{algorithmic}[1]
 \STATE \textbf{Initialization}: $h_a(0) = 1$ for all $a \in [L]$
 \FOR{$t = 1,2,3,\dots$}
    \STATE Observe $a_t, r_t^0$
    \IF {$a_t \neq L$}
        \STATE Calculate $\gamma_t, \tilde{\gamma}_t$ according to \cref{eq:gamma_t,eq:ti_gamma_t}
        \IF{$\gamma_t \le r_t^0$}
            \STATE $\alpha_t =  \lceil \gamma_t  \rceil$, $h_{a_t}(t) = t$
        \ELSE
            \STATE $\alpha_t =  \lceil \tilde{\gamma}_t  \rceil$, $h_{a_t}(t)= h_{a_t}(t-1)$
        \ENDIF
    \ENDIF
    \STATE Return $r_t = r_t^0 - \alpha_t$;  update $h_a(t)= h_a(t-1)$ for all $a\neq a_t$
 \ENDFOR
 \end{algorithmic} 
\end{algorithm}
To overcome these issues, we propose a new attack algorithm against UCB on $L$-armed bandits with binary feedback. It is described in \cref{alg:UCB}. It maintains timestamp
% \mo{can we label this slot with attack timestamp or something like that? Is this the last slot that the attacker corrupts the reward?  }
% \jinhang{It is the last slot that $\gamma_t \le r_t^0$ holds; it is possible that $\gamma_t < r_t^0$ but we still corrupt the reward based on $\tilde{\gamma}_t$.}
$h_a(t)$ for each arm $a$. In round $t$, if the arm pulled by the player, $a_t$, is not the target arm $L$, it first checks the condition $\gamma_t \le r_t^0$ (Line 6), with $\gamma_t$ computed as
\begin{equation}\label{eq:gamma_t}\textstyle
\gamma_t = \left[N_{a_t}(t) \hat{\mu}^0_{a_t}(t) - \sum_{s\in \tau_{a_t}(t-1)} \alpha_s - N_{a_t}(t) \left[\underline{\mu}_L(t) - \Delta_0\right]_{+}\right]_{+},
\end{equation}
where $[z]_{+} = \max(0,z)$ and $\underline{\mu}_L(t) := \hat{\mu}_L(t) - 2\beta(N_L(t))$. In fact, condition $\gamma_t \le r_t^0$ is equivalent to checking whether there exists a feasible $\alpha_t$ to ensure \cref{eq:old_alpha} holds: if $\gamma_t \le r_t^0$, with $\alpha_t$ set to be $\lceil \gamma_t  \rceil$ (Line 7), \cref{eq:old_alpha} will hold in round $t$.
Thus, with a similar analysis to \cite{jun2018adversarial}, we can derive an upper bound of $N_{a_t}(t)$ and prove the success of the attack up to round $t$. The algorithm also updates the timestamp $h_{a_t}(t) = t$ for $a_t$ (line 7). If $\gamma_t > r_t^0$, it indicates that there is no feasible $\alpha_t$ that can ensure \cref{eq:old_alpha}. Instead, the algorithm sets $\alpha_t =  \lceil \tilde{\gamma}_t  \rceil$ with $\tilde{\gamma}_t$ computed as 
\begin{equation}\label{eq:ti_gamma_t}\textstyle
\tilde{\gamma}_t = \left[N_{a_t}(t) \hat{\mu}^0_{a_t}(t) - \sum_{s\in \tau_{a_t}(t-1)} \alpha_s - N_{a_t}(t) \left[\underline{\mu}_L(h_{a_t}(t-1)) - \Delta_0\right]_{+}\right]_{+},
\end{equation}
where $h_{a_t}(t-1)$ records the last round that \cref{eq:old_alpha} was satisfied. We can prove such an $\alpha_t$ is always feasible ($\alpha_t \le r^0_t$) and it can ensure that 
\begin{equation}\label{eq:new_alpha}
    \hat{\mu}_{a_t}(t) \le \hat{\mu}_L(h_{a_t}(t-1)) - 2\beta(N_L(h_{a_t}(t-1))) - \Delta_0.
\end{equation}
This new inequality always holds for all rounds, which helps guarantee the success of the attack.
% \mo{is it possible to have some interpretation of $\gamma_t$ and $\hat{\gamma}_t$? since $r_t^0$ is either 0 or 1, I cannot understand why you need multiple if clause. the only thing matters is when $r_t^0=1$ and in this case you may choose to flip it to 0? if so, what is the rationale behind lines 8 and 9?}\jinhang{The logic here is we first need to check whether there exists a feasible $\alpha_t$ to ensure Eq. (1) and it is equivalent to check $\gamma_t \le r^0_{t}$; if yes, just use that $\alpha_t$; if no, we need to choose the $\alpha_t$ that can ensure Eq. (4) by calculating $\tilde{\gamma}_t$ and that $\alpha_t$ would always be valid. One example about why $\tilde{\gamma}_t$ matters: it is possible that $\gamma_t>1, r^0_{t}=1$ but $\tilde{\gamma}_t = 0$, so it wound not attack though $r^0_{t}=1$.} \mo{maybe we can add part of this explanation to the main text. I mean the first part not my 2nd question.}

\textbf{Remark 1}.
Compared with \cref{eq:old_alpha}, \cref{eq:new_alpha} uses a more conservative lower bound of $\mu_L$, $\underline{\mu}_L(h_{a_t}(t-1))$, instead of $\underline{\mu}_L(t)$, on the right hand side of inequality. We call this \textit{conservative estimation} of the target arm $L$ with respect to arm $a_t$, where the estimated lower bound of $L$ will only be updated when there exists feasible $\alpha_t$ to ensure \cref{eq:old_alpha}.
We use an inductive proof to show that there always exists $\alpha_t$ such that \cref{eq:new_alpha} holds while keeping $r_t$ valid (binary). Hence, the conservative estimation solves the issues introduced by binary feedback. This is then the basis for attack algorithms for OLTR with click feedback presented later in the paper. 

\subsection{Analysis}\label{sec:ucb_analysis}
We first show that our attack algorithm always returns valid binary feedback to the player.
\begin{lemma}\label{lemma:binary}
The post-attack feedback of \Cref{alg:UCB} is always valid, i.e., $r_t \in \{0,1\}$ for any $t$.
\end{lemma}
The proof of \Cref{lemma:binary}, which uses an inductive analysis on $\gamma_t, \tilde{\gamma}_t$, can be found in the appendix.

% {\color{orange}
% \textit{Proof sketch}. 
% When $\gamma_t \le r_t^0$, it is easy to check $\alpha_t = \lceil \gamma_t  \rceil \le r_t^0$, thus $r_t \in \{0,1\}$. When $\gamma_t > r_t^0$, we can prove $\tilde{\gamma}_t \leq r_t^0$ for any $a_t$ by induction. 
% Consider any arm $a \neq L$.   
% We denote $t_{a,j}$ as the $j$-th time that the UCB algorithm played $a$. Since the UCB algorithm plays each arm one time in the beginning, we have $\tilde{\gamma}_{t_{a,1}} \le \left[N_{a}(t_{a,1}) \hat{\mu}^0_{a}(t_{a,1}) \right]_{+} \le r^0_{t_{a,1}}$. Next, we want to show that if $\tilde{\gamma}_{t_{a,j}} \le r^0_{t_{a,j}}$, $\tilde{\gamma}_{t_{a,j+1}} \le r^0_{t_{a,j+1}}$. We consider two cases for $t_{a,j}$: 1) $\gamma_{t_{a,j}} \le r_{t_{a,j}}^0$; 2) $\gamma_{t_{a,j}} > r_{t_{a,j}}^0$. In case 1, we can bound $\tilde{\gamma}_{t_{a,j+1}} \le r^0_{t_{a,j+1}}$ using $\gamma_{t_{a,j}}$; in case 2, we  bound $\tilde{\gamma}_{t_{a,j+1}} \le r^0_{t_{a,j+1}}$ using $\tilde{\gamma}_{t_{a,j}}$. Since $\tilde{\gamma}_{t_{a,1}} \le r^0_{t_{a,1}}$, by induction, we have $\tilde{\gamma}_{t_{a,j}} \le r^0_{t_{a,j}}$ for any $a, j$, which concludes the proof.
% }

Define $\Delta_a := \mu_a - \mu_L$.
We give the following theorem to show the successful attack of \Cref{alg:UCB}. 
\begin{theorem}\label{thm:binary}
Suppose $T \ge L$ and $\delta \le 1/2$. With probability at least $1 - \delta$, \Cref{alg:UCB} misguides the UCB algorithm to choose the target arm $L$ at least
$T - (L-1)\left(1+\frac{3}{\Delta_0^2} \log T\right)$ rounds, using a a cumulative attack cost at most
\begin{equation*}\textstyle
C(T) \le \left(1+\frac{3}{\Delta_0^2} \log T\right) \sum_{a<L}\left(\Delta_a + \Delta_0 + 4\beta\left(1+\frac{3}{\Delta_0^2} \log h_a(T)\right)\right).
\end{equation*}
As $T$ goes to infinity, we have
\begin{equation*}
\lim_{T \rightarrow \infty} \frac{C(T)}{\log T}
\le O\left(\sum_{a<K}\frac{\Delta_a + \Delta_0}{\Delta_0^2}\right).
\end{equation*}
\end{theorem}
Compared with Theorem 2 in \cite{jun2018adversarial}, the $\beta$ term in our cost bound depends on $\log h_a(T)$ instead of $\log T$. Since $\beta$ is a decreasing function and $h_a(T) \le T$, our finite-time cost bound can be larger than that in \cite{jun2018adversarial}. However, our asymptotic analysis of the cost suggests that when $T$ is large, such difference becomes negligible. 
Notice that the attack algorithm in \cite{jun2018adversarial} does not have any theoretical guarantee in the binary feedback setting, so this comparison is only meant to show the additional cost potentially caused by the conservative estimation.

%% file: 4_oltr.tex
\section{Attacks on Online Learning to Rank}\label{sec:oltr}
We now move to developing effective attack algorithms for more complicated OLTR settings. Since attacking OLTR also faces the challenges caused by binary click feedback, the attack algorithms in this section rely on our attack design for stochastic bandits with binary feedback in \Cref{sec:binary}.

\subsection{Position-Based Click Model}
\textbf{Threat Model}.
We introduce the threat model for online stochastic ranking with position-based click feedback. In each round $t$, the player chooses a list of $K$ item, $\bm{a}_t=(a_{1,t},\cdots, a_{K,t})$ to recommend. The environment generates the pre-attack click feedback $\bm{r}^0_t = (r^0_{1,t},\cdots,r^0_{K,t})\in \{0,1\}^{K}$ where $r^0_{i,t}$ is sampled from Bernoulli distribution with mean $\kappa_{i} \mu_{a_{i,t}}$. 
% \mo{define Bern, what about $\mathcal{B}$ as notation?} \jinhang{revised}
The attacker then observes $\bm{a}_t$ and $\bm{r}^0_t$, and decides the post-attack click feedback $\bm{r}_t \in \{0,1\}^{K}$. The player only receives $\bm{r}_t$ as the feedback and uses it to decide the next list to recommend, $\bm{a}_{t+1}$, for round $t+1$. Without loss of generality, we assume item $L$ is a sub-optimal target item. Similar to \Cref{sec:binary_model}, we say the attack is successful after $T$ rounds if the number of target item recommendations is $N_L(T) = T - o(T)$ in expectation or with high probability, while the cumulative attack cost $C(T) = \sum_{t=1}^{T}||\bm{r}_t - \bm{r}^0_t||_1 = o(T)$. 
% \jinhang{found $C(T)$ was used in Table 1, so I change all $\sum_t \alpha_t$ in theorems to $C(T)$}
% \mo{are you using $C(T)$ later?} \mo{how different is to have more than one target arm but $\leq K$?} \jinhang{We can actually remove $C(T)$. Our algorithms still work when more than one target arm but $<K$. Just need to put the target arms into $\bm{a}^*$. Maybe we can claim this somewhere after the algorithm design?} \mo{may be worth mentioning.}

\textbf{Attack against PBM-UCB}.
We consider the PBM-UCB algorithm in \cite{lagree2016multiple} as the online ranking algorithm of the player, which computes the UCB index of each item $a$ as
\begin{equation}\textstyle\label{eq:PBM-UCB-value}
    \bar{\mu}_a(t) = \hat{\mu}_{a}(t-1) + B_{a}(t) = \hat{\mu}_{a}(t-1) + \sqrt{\frac{N_a(t-1)(1+\epsilon)\log t}{2\tilde{N}_a(t-1)^2}},
\end{equation}
where $\tilde{N}_a(t) := \sum_{s=1}^{t}\sum_{i=1}^{K}\kappa_{i}I\{a_{i,s} = a\}$ is the position bias-corrected counter, $\hat{\mu}_{a}(t)$ is the empirical mean of item $a$, and $\epsilon$ is a parameter of the algorithm. The algorithm then chooses the corresponding first $K$ items with the highest UCB indices as the recommendation list.

\begin{algorithm}[t]
 \caption{Attack against the PBM-UCB algorithm}\label{alg:PBM-UCB}
 \begin{algorithmic}[1]
 \STATE \textbf{Initialization}: Randomly select $K-1$ items with $L$ to build $\bm{a}^*$; $h_{l,a}(0) = 1\,\, \forall l \in [L]\,\,\forall a \in \bm{a}^*$
 \FOR{$t = 1,2,3,\dots$}
    \STATE Observe $\bm{a}_t, \bm{r}_t^0$; set $\bm{\alpha}_t = (\alpha_{1,t}, \cdots, \alpha_{K,t}) = (0,\cdots,0)$
    \FOR{$i \in [K]$}
        \IF {$a_{i,t} \notin \bm{a}^*$}
            \STATE $\alpha_{i,t} = \text{\texttt{CAL\_ALPHA}($a_{i,t}$, $r^0_{i,t}$, $\bm{a}^*$)}$ 
        \ELSE
            \STATE $\alpha_{i,t} = 0$
        \ENDIF
    \ENDFOR
    \STATE Return $\bm{r}_t = \bm{r}_t^0 - \bm{\alpha}_t$; update $h_{l, a}(t) = h_{l, a}(t-1)$ for all $l \notin \bm{a}_t, a\in \bm{a}^*$
 \ENDFOR
 \end{algorithmic} 
\end{algorithm}

\begin{algorithm}[t]
 \caption{\texttt{CAL\_ALPHA}}\label{alg:cal_alpha}
 \begin{algorithmic}[1]
 \STATE \textbf{Input}: item $l$, click feedback $r^0$, item set $\bm{a}^*$
    \FOR{$a \in \bm{a}^*$}
    \STATE Calculate $\gamma_t(l,a), \tilde{\gamma}_t(l,a)$ according to \cref{eq:gamma_la,eq:ti_gamma_la}
    % \IF{$\gamma_t(l,a) \le r^0$}
    %     \STATE $h_{l,a}(t) = t$
    % \ELSE
    %     \STATE $h_{l,a}(t) = h_{l,a}(t-1)$
    % \ENDIF
    \ENDFOR
 \STATE $\gamma_{\max} = \max_{a \in \bm{a}^*} \gamma_{t}(l,a), \tilde{\gamma}_{\max} = \max_{a \in \bm{a}^*} \tilde{\gamma}_{t}(l,a)$
 \IF{$\gamma_{\max}  \le r^0$}
    \STATE $\alpha =  \lceil \gamma_{\max}  \rceil$, $h_{l,a}(t) = t$ for all $a \in \bm{a}^*$
\ELSE
    \STATE $\alpha =  \lceil \tilde{\gamma}_{\max} \rceil$, $h_{l,a}(t) = h_{l,a}(t-1)$ for all $a \in \bm{a}^*$
\ENDIF
\STATE \textbf{Return} $\alpha$
\end{algorithmic} 
\end{algorithm}

We propose our attack algorithm against PBM-UCB in \Cref{alg:PBM-UCB}. It works by first randomly taking $K-1$ items out and making them a set with the target item $L$, denoted as $\bm{a}^* = \{a^*_1, \cdots, a^*_{K-1}, L\}$. Then, based on the conservative estimation idea from \Cref{alg:UCB}, it maintains timestamp $h(l,a)$ for each item $l$ with respect to each $a \in \bm{a}^*$. The intuition is that, to ensure a similar inequality as \cref{eq:new_alpha} for all rounds, we need to make \textit{conservative estimation} on the lower bounds of $\mu_{a}$ for all $a \in \bm{a}^*$. This is handled by \Cref{alg:cal_alpha}, which maintains the timestamps $h_{l,a}(t)$ for the input item $l$ and outputs the appropriate attack value $\alpha$ on $l$ that can always ensure the required inequality.  The value of parameters $\gamma_{t}(l,a)$ and $\tilde{\gamma}_{t}(l,a)$ in \Cref{alg:cal_alpha} are computed as
\begin{equation}\label{eq:gamma_la}\textstyle
\gamma_t(l,a) = \left[N_{l}(t) \hat{\mu}^0_{l}(t) - \sum_{s\in \tau_{l}(t-1)} \alpha_{l}{(s)} - N_{l}(t) \left[\underline{\mu}_a(t) - \Delta_0\right]_{+}\right]_{+},
\end{equation}

\begin{equation}\label{eq:ti_gamma_la}\textstyle
\tilde{\gamma}_t(l,a) = \left[N_{l}(t) \hat{\mu}^0_{l}(t) - \sum_{s\in \tau_{l}(t-1)} \alpha_{l}{(s)} - N_{l}(t) \left[\underline{\mu}_a(h_{l,a}(t-1)) - \Delta_0\right]_{+}\right]_{+}.
\end{equation}
Notice that \Cref{alg:PBM-UCB} could also work when there are more than one but less than $K+1$ target arms (the goal of the attacker becomes misguiding the player to recommend all target arms very often with sublinear cost). The only modification required is to put all of these target arms into $\bm{a}^*$.

% \textbf{Analysis}. Similar to the $L$-armed setting, we can prove that the post-attack feedback of \Cref{alg:PBM-UCB} is always valid.
% \begin{lemma}
% The post-attack feedback of \Cref{alg:PBM-UCB} is always valid, i.e., $\bm{r}_{t} \in \{0,1\}^{M}$ for any $t$.
% \end{lemma}
\textbf{Analysis}.
The following theorem shows the attack against PBM-UCB is successful.
\begin{theorem}\label{thm:pbm}
Suppose $T \ge L$ and $\delta \le 1/2$. With probability at least $1 - \delta$, \Cref{alg:PBM-UCB} misguides the PBM-UCB algorithm to recommend the target item $L$ at least
$T - (L-K)\left(\frac{1+\epsilon}{2 \kappa^2_{K}\Delta_0^2} \log T\right)$ rounds, using a cumulative attack cost at most
\begin{equation*}\textstyle
C(T) \le \left(\frac{1+\epsilon}{2 \kappa^2_{K}\Delta_0^2} \log T\right)\sum_{a<L}\left(\Delta_a + \Delta_0 + 4\beta\left(\frac{1+\epsilon}{2 \kappa^2_{K}\Delta_0^2} \log h_{a,L}(T)\right)\right).    
\end{equation*}
When $T$ goes to infinity, we have
$$ \lim_{T \rightarrow \infty} \frac{C(T)}{\log T}
\le O\left(\sum_{a<K}\frac{(1+\epsilon)(\Delta_a + \Delta_0)}{\kappa^2_{K}\Delta_0^2}\right).$$
\end{theorem}
\textit{Proof sketch}. Whenever $a_{i,t} \notin \bm{a}^*$ is chosen by the player, there must exist $a \in \bm{a}^*$ such that $\bar{\mu}_{a_{i,t}}(t) \ge \bar{\mu}_{a}(t)$. The output $\alpha_{i,t}$ of \Cref{alg:cal_alpha} would ensure $\hat{\mu}_{a_{i,t}}(t) \le \hat{\mu}_{a}(h_{a_{i,t}, a}(t-1)) - 2\beta(N_a(h_{a_{i,t}, a}(t-1))) - \Delta_0$, owing to the conservative estimations in $\gamma_t(a_{i,t},a)$ and $ \tilde{\gamma}_t(a_{i,t},a)$. Combining these two inequalities, we can get $B_{a_{i,t}}(t) - B_{a}(t) \ge \Delta_0$. With a careful calculation on this inequality involved bias-corrected counters, we have $N_{a_{i,t}}(t) \le \left(\frac{1+\epsilon}{2 \kappa^2_{K}\Delta_0^2} \log t\right)$. This result holds for any $a_{i,t} \notin \bm{a}^*$. Thus, we immediately get the bound of $N_L(t)$. The remaining proof for the cost bound will be similar to that of \Cref{thm:binary}. 
% \mo{either as part of the proof sketch or as a separate remark after the theorem explain the additional challenge w.r.t, standard attack with binary feedback. also, we may discuss the significance of this result. sort of explaining with log attack cost the attacker corrupts the alg and the target arm will be selected too many times such that alg fails to achieve sublinear regret. }

Compared with \Cref{thm:binary}, the asymptotic cost of \Cref{alg:PBM-UCB} has an additional dependency on $1/ \kappa_K^2$, which suggests it may require more cost to achieve a successful attack in the PBM model, though the cost dependency on $T$ is still logarithmic.

\subsection{Cascade Click Model}
% \begin{enumerate}
%     \item Bandit algorithm recommend $K$ out of $L$ items, $\bm{a}_t=(a_{1,t},\cdots, a_{K,t})$
%     \item User generate pre-attack click feedback $\bm{r}^0_t \in \{0,1\}^{K}, ||\bm{r}^0_t|| \le 1$
%     \item Attacker observes $\bm{a}_t, \bm{r}^0_t$, and decides post-attack feedback $\bm{r}_t \in \{0,1\}^{K}, ||\bm{r}_t|| \le 1$
%     \item Bandit algorithm receives $\bm{r}_t$ and then chooses the next arm $\bm{a}_t$
% \end{enumerate}
% For any target item $k$, we say the attack is successful after $T$ rounds if:
% \begin{itemize}
%     \item $\mathbb{E}[N_k(T)] = T - o(T)$ (\# of recommendataions, not \# of samples)
%     \item $\mathbb{E}[C(T)] = \mathbb{E}[\sum_{t=1}^{T}||\bm{r}_t - \bm{r}^0_t||_1] = o(T)$
% \end{itemize}
\textbf{Threat Model}.
We introduce the threat model for OLTR with cascade click feedback. In each round $t$, the player chooses a list of $K$ item, $\bm{a}_t=(a_{1,t},\cdots, a_{K,t})$ to recommend. The environment generates the pre-attack click feedback $\bm{r}^0_t = (r^0_{1,t},\cdots,r^0_{K,t})\in \{0,1\}^{K}$, $||\bm{r}^0_t||\le 1$. Let $\textbf{C}_t$  denote the position of the clicked item, i.e., $r^0_{\textbf{C}_t,t} = 1$ ($\textbf{C}_t = \infty$ if none item was clicked). The attacker observes $\bm{a}_t, \bm{r}^0_t, \textbf{C}_t$, and decides the post-attack click feedback $\bm{r}_t \in \{0,1\}^{K}$, $||\bm{r}_t||_1 \le 1$. 
% \mo{isn't too strong to assume that the attacker sees $\textbf{C}_t$? and if this assumption is ok, why the attacker should bother changing any other reward but that of $\textbf{C}_t$?}
% \jinhang{$\textbf{C}_t$ can directly be observed from $\bm{r}^0_t$ (just the position of the first non-zero item). We do only attack $\textbf{C}_t$, but setting it to zero would misguide the learning agent to keep checking the items after $\textbf{C}_t$ (all zeros); we need to set the item in $\bm{a}^*$ after $\textbf{C}_t$ to 1 so that it will not be under-estimated by the agent.}\jinhang{I forgot to mention this in the attack design; adding it now} 
The player only receives $\bm{r}_t$ as the feedback and uses it to decide the next list to recommend, $\bm{a}_{t+1}$, for round $t+1$. The goal of the attacker in this setting is the same as that in the PBM model.

\begin{algorithm}[t]
 \caption{Attack against the CascadeUCB algorithm}\label{alg:CascadeUCB}
 \begin{algorithmic}[1]
 \STATE \textbf{Initialization}: Randomly select $K-1$ items with $L$ to build $\bm{a}^*$; $h_{l,a}(0) = 1\,\, \forall l \in [L]\,\,\forall a \in \bm{a}^*$
 \FOR{$t = 1,2,3,\dots$}
    \STATE Observe $\bm{a}_t, \bm{r}_t^0, \textbf{C}_t$; set $\bm{\alpha}_t = (\alpha_{1,t}, \cdots, \alpha_{K,t}) = (0,\cdots,0)$
    \IF{$\textbf{C}_t \le K$ and $a_{\textbf{C}_t,t} \notin \bm{a}^*$}
        \STATE $\alpha_{\textbf{C}_t,t} = \texttt{CAL\_ALPHA}(a_{\textbf{C}_t,t}, r^0_{\textbf{C}_t,t}, \bm{a}^*)$
        \IF{$\alpha_{\textbf{C}_t,t} = 1$ and $\exists i>C_{t}$ s.t. $a_{i,t} \in \bm{a^*}$}
            \STATE $\alpha_{i,t} = -1$  
        \ENDIF
    \ENDIF
    \STATE Return $\bm{r}_t = \bm{r}_t^0 - \bm{\alpha}_t$; update $h_{l, a}(t) = h_{l, a}(t-1)$ for all $l \neq a_{\textbf{C}_t,t}, a\in \bm{a}^*$
 \ENDFOR
 \end{algorithmic} 
\end{algorithm}

\textbf{Attack against CascadeUCB.}
We propose an attack algorithm against CascadeUCB in \Cref{alg:CascadeUCB}. Same as the attack against PBM-UCB, it first randomly generates a set of items $\bm{a}^* = \{a^*_1, \cdots, a^*_{K-1}, L\}$.
It also follows the idea of conservative estimation: when the clicked item $a_{\textbf{C}_t, t}$ does not belong to $\bm{a}^*$, it calls \Cref{alg:cal_alpha} to maintain $h_{a_{\textbf{C}_t, t} ,a}$ for all $a \in \bm{a}^*$ and compute the attack value $\alpha_{\textbf{C}_t, t}$ based on the conservative estimations. If the output $\alpha_{\textbf{C}_t, t}=1$, which means the algorithm sets the clicked position to be zero, the player will keep checking the positions after $\textbf{C}_t$. Since the pre-attack feedback of all items after position $\textbf{C}_t$ is zero, we need to find the first item $a_{i,t} \in \bm{a}^*$ after position $\textbf{C}_t$ and set $\alpha_{i,t} = -1$ ($r_{i,t} = 1$). The empirical means of the items between position $\textbf{C}_t$ and $i$ that are not in $\bm{a}^*$ decreases, and the empirical mean of $a_{i,t} \in \bm{a}^*$ increases. Thus it does not affect the success of the attack while still making the post-attack feedback $\bm{r}_t$ to be valid.

\textbf{Analysis}. 
% \mo{can we discuss a bit about the technical difference between pbm vs. cascade model that makes this section independently interesting?}\jinhang{maybe it's better to put the discussion on mismatch between recommendation and observation in the proof sketch here.}
First, note that there is a mismatch between recommendation and observation in the cascade model: for all $i$ such that $\textbf{C}_t < i \le K$, $a_{i,t}$ is recommended but not observed, i.e., there is no new click sample for item $a_{i,t}$ for estimation. We can still follow a similar proof of \Cref{thm:pbm} to get the upper bound of the number of observations (samples) for $a_{i,t}\notin \bm{a^*}$, but it is less than the number of recommendations thus cannot ensure the success of the attack. To tackle this problem, we use a new regret-based analysis on the expected number of recommendations. 
% \begin{lemma}
% The post-attack feedback of \Cref{alg:CascadeUCB} is valid, i.e., $\bm{r}_{t} \in \{0,1\}^{K}, ||\bm{r}_{t}||_1 \le 1$ for any $t$.
% \end{lemma}
We can also prove that the post-attack feedback of \Cref{alg:CascadeUCB} is always valid.
Define $p^* := \prod_{i=1}^{K-1} \mu_{i}$. We give the following theorem of the successful attack against CascadeUCB.
\begin{theorem} \label{thm:cascade}
Suppose $T \ge L$ and $\delta \le 1/2$. With probability at least $1 - \delta$, \Cref{alg:PBM-UCB} misguides the CascadeUCB algorithm to choose the target arm at least
$T - (L-K)\left(\frac{12}{p^{*}\Delta_0^2} \log T\right)$ rounds in expectation. Its cumulative attack cost at most
\begin{equation*}\textstyle
C(T) \le \left(1+\frac{3}{\Delta_0^2} \log T\right) \sum_{a<K}\left(\Delta_a + \Delta_0 + 4\beta\left(1+\frac{3}{\Delta_0^2} \log h_{a,L}(T)\right)\right).    
\end{equation*}
As $T$ goes to infinity, we have
$$ \lim_{T \rightarrow \infty} \frac{C(T)}{\log T}
\le O\left(\sum_{a<K}\frac{\Delta_a + \Delta_0}{\Delta_0^2}\right).$$
\end{theorem}
\textit{Proof sketch}. As mentioned above, we use a new regret-based analysis. The intuition is that the regret caused by any suboptimal item is the product of its reward gap with respect to the optimal list and its expected number of recommendations. If we know the regret upper bound and the reward gap lower bound, we can derive the upper bound of the expected number of recommendations. To do so, we first show that the post-attack problem can be viewed as a problem with known reward gap lower bound. This can be obtained by $\hat{\mu}_{a_{i,t}(t)} \le \hat{\mu}_{L}(h_{L_{i,t}, a}(t-1)) - 2\beta(N_a(h_{a_{i,t}, L}(t-1))) - \Delta_0$ for any $a_{i,t} \notin \bm{a}^*$, which indicates the post-attack expected reward gap between $a_{i,t}$ and $L$ is always larger than $\Delta_0$. Then, the lower bound of the reward gap of $a_{i,t}$ with respect to the optimal list is $p^* \Delta_0$. Based on the regret upper bound $12\log T/ \Delta_0 $ given in \cite{kveton2015cascading}, we can get the upper bound of the expected number of recommendations is $\frac{12}{p^{*}\Delta_0^2} \log T$. For the remaining cost analysis, since the attack cost only depends on the observed items, we can still follow the proof of \Cref{thm:pbm}.

\begin{algorithm}[t]
 \caption{Attack against arbitrary algorithm}\label{alg:general}
 \begin{algorithmic}[1]
 \STATE \textbf{Initialization}: Randomly select $K-1$ items with target $L$ to build $\bm{a}^*$
 \FOR{$t = 1,2,3,\dots$}
    \STATE Observe $\bm{a}_t, \bm{r}_t^0$
    \FOR{$i \in [K]$}
        \STATE $\alpha_{i,t} = 0$
        \IF{$a_{i,t} \notin \bm{a}^*$ and $r^0_{i,t} = 1$}
            \STATE $p_{i,t} = \max_{a\in \bm{a}^*} \frac{\left[\hat{\mu}^0_{a_{i,t}}(t) + \beta(N_{a_{i,t}}(t)) - \hat{\mu}^0_{a}(t) + \beta(N_{a}(t))\right]_{+}}{\hat{\mu}^0_{a_{i,t}}(t) + \beta(N_{a_{i,t}}(t))}$
            \STATE With prob. $p_{i,t}$, set $\alpha_{i,t} = 1$
        \ENDIF
    \ENDFOR
    \STATE Return $\bm{r}_t = \bm{r}_t^0 - \bm{\alpha}_t$
 \ENDFOR
 \end{algorithmic} 
\end{algorithm}

\subsection{General Attacks on OLTR with General Click Model}
We have provided attack strategies against UCB-based OLTR algorithms under two specific click models. A natural follow-up question is whether there is any attack strategy that can attack any OLTR algorithm under the general click model. To answer this question in what follows, we design an attack strategy that can misguide any OLTR algorithm without knowing the underlying algorithm. However, it may pay more cost (still sublinear) than the others since it cannot take advantage of the details of the algorithm to make fine-tuned adjustments.

\textbf{Threat Model}.
We consider the threat model for OLTR with general click feedback. In each round $t$, the player chooses a list of $K$ item, $\bm{a}_t=(a_{1,t},\cdots, a_{K,t})$ to recommend. The environment generates the pre-attack click feedback $\bm{r}^0_t = (r^0_{1,t},\cdots,r^0_{K,t})\in \{0,1\}^{K}$, $\bm{r}^0_t \in \mathcal{R}_c$, where $\mathcal{R}_c$ is the feasible feedback space of click model $c$. 
% \mo{can we define $R_c$ of either cascade or pbm to see how this is general?}\jinhang{it is a bit hard to define cuz the click model $c$ also decides how to generate $r^0_t$ but not only the feasible feedback space; like for PBM, $\mathcal{R}_c$ is just $\{0,1\}^K$ but the way to generate $r^0_t$ can be different from other models.}
The attacker observes $\bm{a}_t, \bm{r}^0_t$, and decides the post-attack click feedback $\bm{r}_t \in \{0,1\}^{K}$, $\bm{r}_t \in \mathcal{R}_c$.
Notice that the attack should know $\mathcal{R}_c$, otherwise, it is impossible to ensure valid post-attack feedback.
The player only receives $\bm{r}_t$ as the feedback and uses it to decide the next list to recommend, $\bm{a}_{t+1}$. The goal of the attacker is the same as that in PBM model.

\textbf{General Attack against Arbitrary OLTR Algorithm}.
We propose an attack algorithm against arbitrary OLTR algorithm in \Cref{alg:general}. Same as the attack against PBM-UCB, it first randomly generate a set of items $\bm{a}^* = \{a^*_1, \cdots, a^*_{K-1}, L\}$. In each round, for each clicked item $a_{i,t} \notin \bm{a}^*$, the algorithm calculates an attack probability $p_{i,t}$ and use that to decide whether its feedback needs to be changed to unclicked ($r_{i,t} = 0$). We prove that $p_{i,t}$ is actually an estimated upper bound of $\Delta_i / \mu_i$, thus by such probabilistic feedback perturbations, the algorithm makes all items outside $\bm{a}^*$ be worse than the items inside $\bm{a}^*$, which guarantees the success of the attack.
\begin{theorem}\label{thm:general}
Suppose $T \ge L$ and $\delta \le 1/2$. With probability at least $1 - \delta$, \Cref{alg:general} misguides arbitrary OLTR algorithm that chooses sub-optimal items at most $O(\log T)$ rounds to choose the target item at least
$T - O(\log T)$ rounds in expectation. Its cumulative attack cost is at most
\begin{equation*}\textstyle
C(T) \le O\left(\sum_{a<L}(\Delta_a + 4\beta(1)) \log T \right).    
\end{equation*}
\end{theorem}
Compared with the results in the PBM and cascade models, the $\beta()$ term in the cumulative cost of \Cref{alg:general} is $\beta(1)$, which can be much larger than the others. Also, the asymptotic costs of \Cref{alg:PBM-UCB} and \Cref{alg:CascadeUCB} is independent of $\beta$'s, while the asymptotic cost of \Cref{alg:general} depends on $\beta$, showing that \Cref{alg:general} is more costly than attack strategies specific to OLTR algorithms. 

%% file: 5_experiment.tex
\section{Experiments}\label{sec:experiment}
\begin{figure}[t]
	\centering
 	\begin{subfigure}[b]{0.24\textwidth}
    	\centering
    	\includegraphics[width=\textwidth]{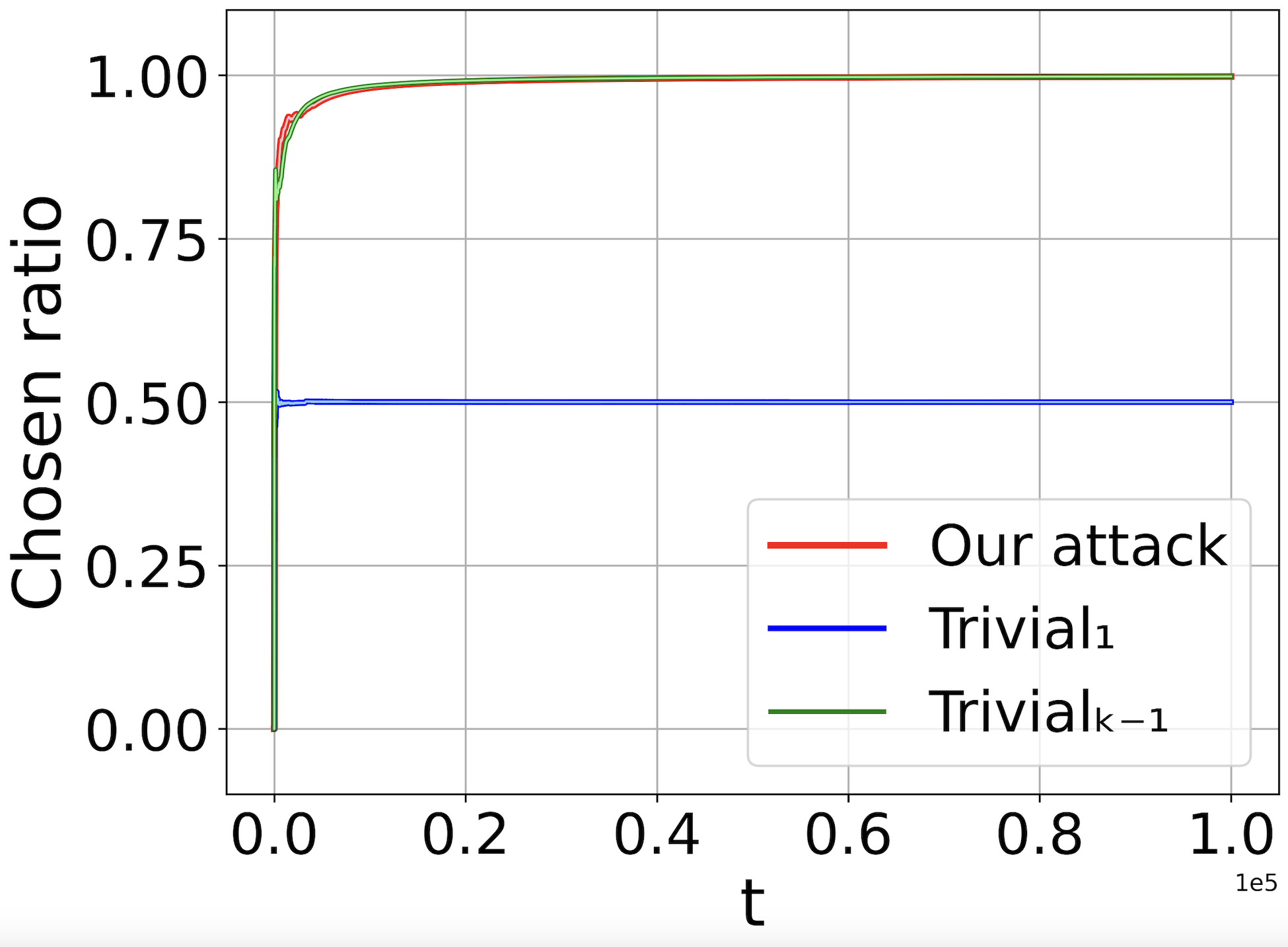}
    	\caption{Chosen ratio (syn.)}
    	\label{fig:SynData_PBM_Ratio}
	\end{subfigure}
	\begin{subfigure}[b]{0.235\textwidth}
		\centering
		\includegraphics[width=\textwidth]{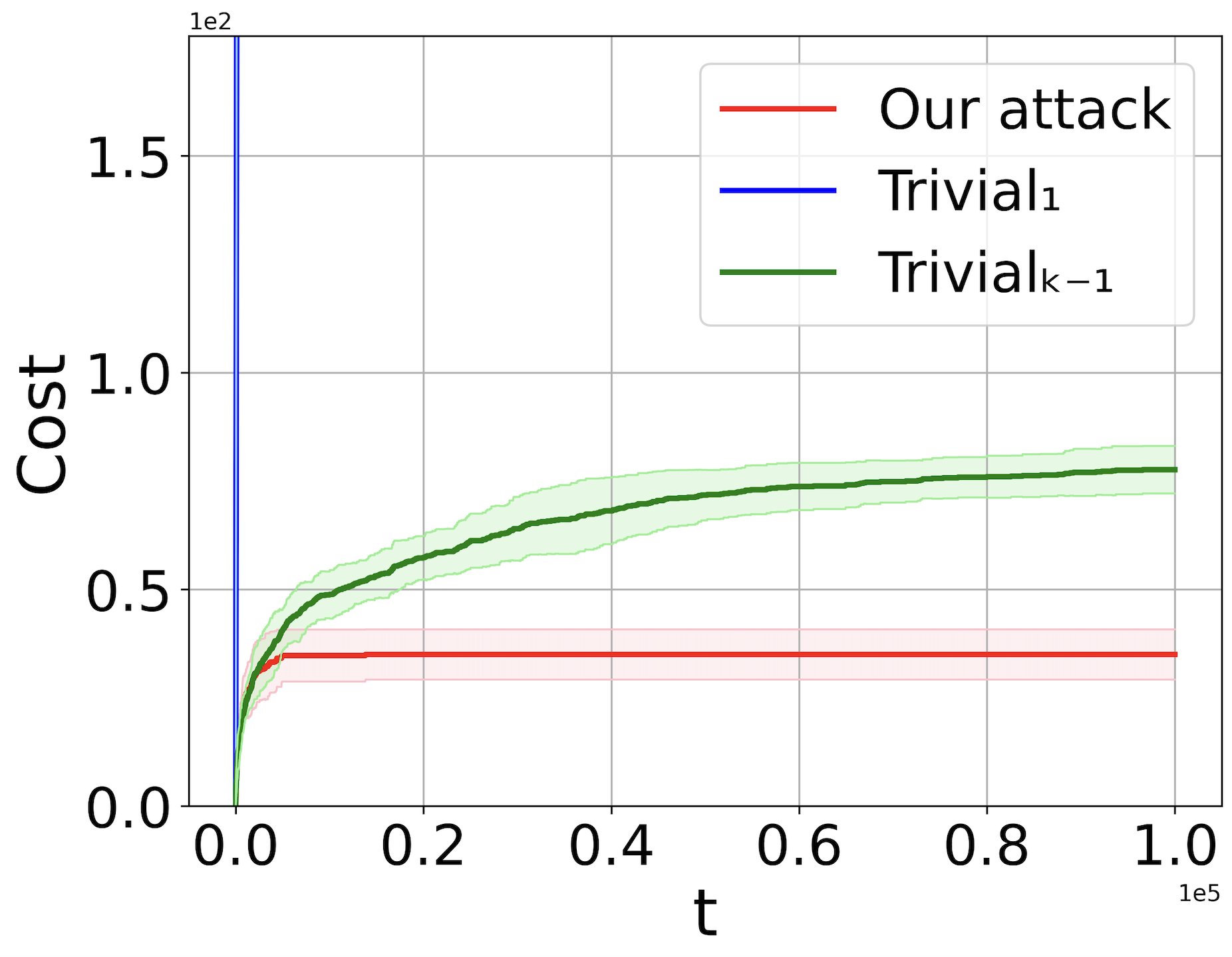}
		\caption{Cost (syn.)}
		\label{fig:SynData_PBM_Cost}
	\end{subfigure}
 	\begin{subfigure}[b]{0.24\textwidth}
		\centering
		\includegraphics[width=\textwidth]{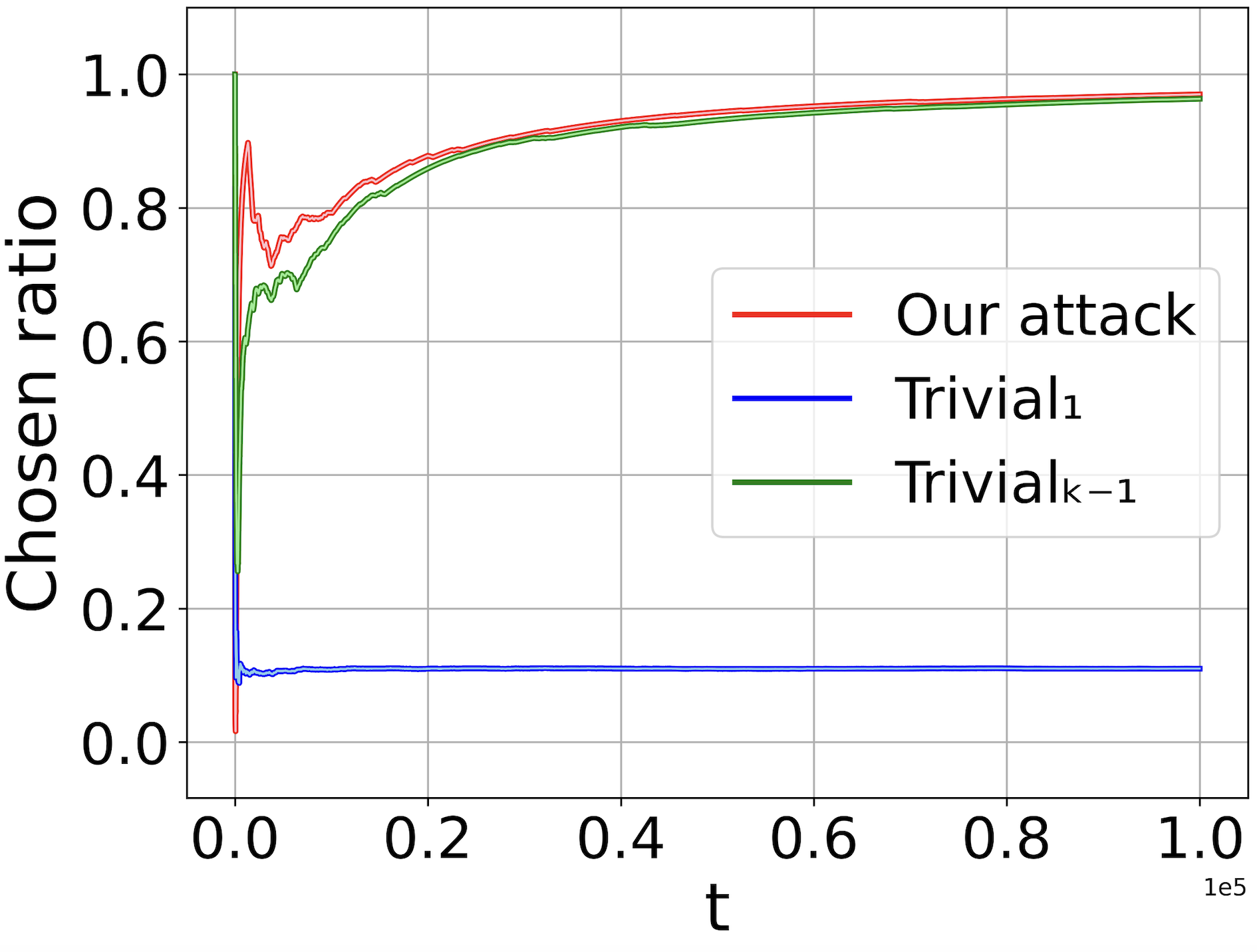}
		\caption{Chosen ratio (real)}
		\label{fig:RealData_PBM_Ratio}
	\end{subfigure}
 	\begin{subfigure}[b]{0.235\textwidth}
		\centering
		\includegraphics[width=\textwidth]{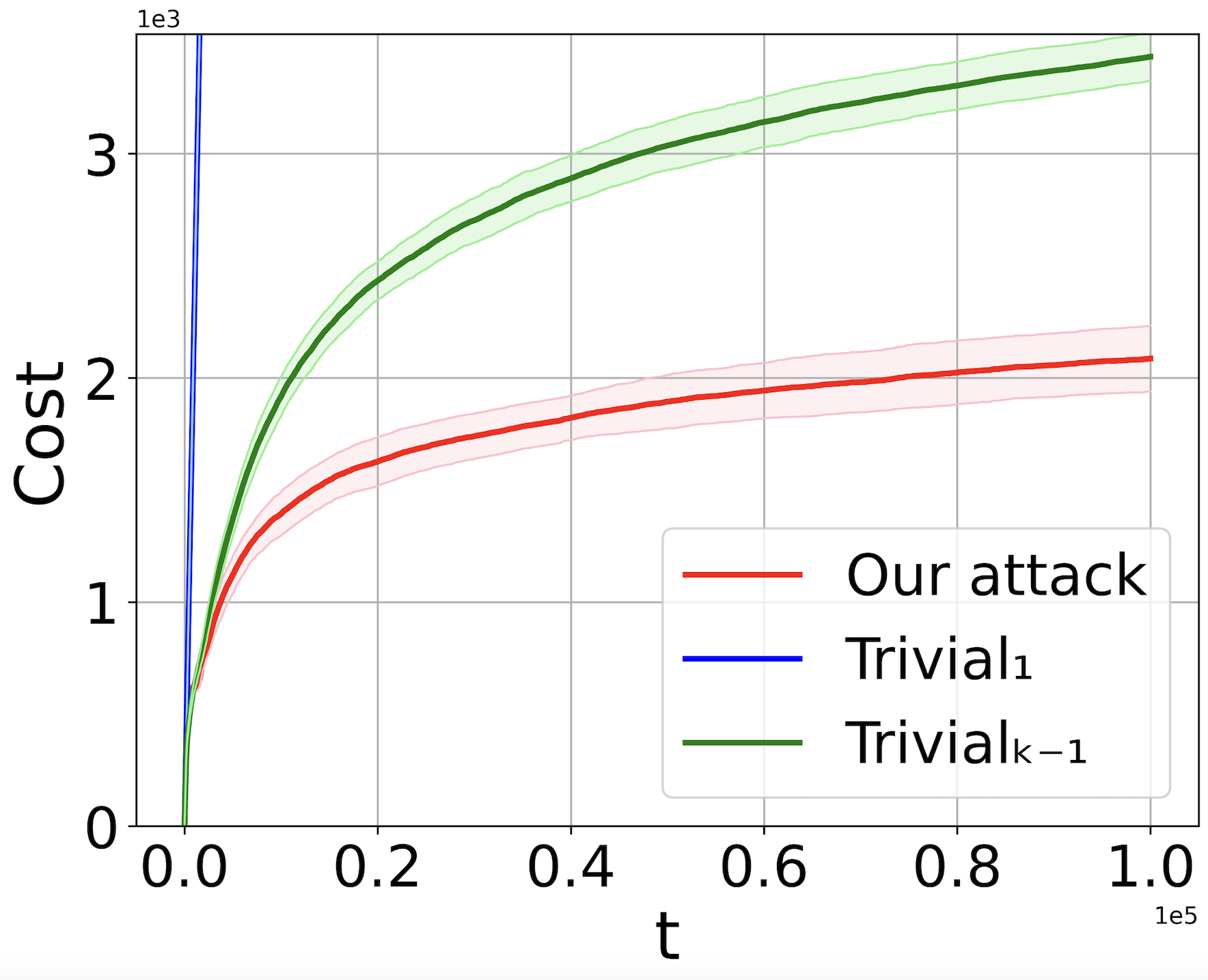}
		\caption{Cost (real)}
		\label{fig:RealData_PBM_Cost}
	\end{subfigure}
	\caption{Attacks against PBM-UCB.}
	\label{fig:RealData_PBM}
\end{figure}
We conduct experiments on both synthetic and real data (MovieLens 20M dataset~\cite{movielens}). Due to the space limit, we only report the results of OLTR with the position-based model. We use $\epsilon = 0.1$ for the PBM-UCB algorithm.
For the synthetic data, we take $L = 16, K = 8, T = 100,000$; $\{\mu_i\}_{i=1}^{L}$ are sampled from uniform distribution $U(0, 1)$. For the real data, we take $L = 100, K = 10, T = 100,000$; $\{\mu_i\}_{i=1}^{L}$ are extracted according to on~\cite{li2019online}. 
We compare our algorithm with two baselines: $\texttt{Trivial}_1$ attacks all arms except the target arm as long as the attack is valid; $\texttt{Trivial}_{K-1}$ first randomly takes $K-1$ arms out to generate a set $\bm{a}^*$ as our algorithm and then attacks all arms outside the set as long as the attack is valid. \Cref{fig:RealData_PBM} shows that $\texttt{Trivial}_{1}$ algorithm cannot successfully attack PBM-UCB even with linear costs. 
Our algorithm and $\texttt{Trivial}_{K - 1}$ have similar performance on the chosen ratio of the target arm as shown in \Cref{fig:SynData_PBM_Ratio,fig:RealData_PBM_Ratio}, which is the chosen time of the target arm divided by the current round. However, our algorithm pays $50\%$ and $40\%$ less cost than $\texttt{Trivial}_{K - 1}$ in \Cref{fig:SynData_PBM_Cost,fig:RealData_PBM_Cost} respectively, which validates the necessity of the attack design.

%% file: appendix.tex
\clearpage
\appendix
\section*{Appendix}
\section{Proofs}
\subsection{Proof of \texorpdfstring{\Cref{lemma:binary}}{L1}}
\begin{proof}
When $\gamma_t \le r_t^0$, it is easy to check $\alpha_t = \lceil \gamma_t  \rceil \le r_t^0$, thus $r_t \in \{0,1\}$. 

When $\gamma_t > r_t^0$, if $\tilde{\gamma}_t \leq r_t^0$, we will have $\alpha_t = \lceil \tilde{\gamma}_t  \rceil \le r_t^0$ and $r_t \in \{0,1\}$. Since $\tilde{\gamma}_t$ depends on $a_t$, our goal is to show $\tilde{\gamma}_t \leq r_t^0$ for any $a_t$. Consider any arm $a \neq K$.   
We denote $t_{a,j}$ as the $j$-th time that the UCB algorithm played $a$. Since the UCB algorithm would play each arm one time in the beginning, we have
\begin{equation}
    \tilde{\gamma}_{t_{a,1}} \le \left[N_{a}(t_{a,1}) \hat{\mu}^0_{a}(t_{a,1}) \right]_{+} = \left[r^0_{t_{a,1}}\right]_{+} \le r^0_{t_{a,1}}.
\end{equation}
Next, we want to show that if $\tilde{\gamma}_{t_{a,j}} \le r^0_{t_{a,j}}$, $\tilde{\gamma}_{t_{a,j+1}} \le r^0_{t_{a,j+1}}$. Since $\tilde{\gamma}_{t}$ also depends on $h_{a_t}(t-1)$, we consider two cases:

1) $\gamma_{t_{a,j}} \le r_{t_{a,j}}^0$. In this case, we have $h_{a}(t_{a,j+1}-1) = t_{a,j}$. We can bound $\tilde{\gamma}_{t_{a,j+1}}$ using $\gamma_{t_{a,j}}$:
\begin{equation}
    \tilde{\gamma}_{t_{a,j+1}} \le \gamma_{t_{a,j}} + r^0_{t_{a,j+1}} - \alpha_{t_{a,j}} - \left[\underline{\mu}_L(t_{a,j}) - \Delta_0\right]_{+} \le  r^0_{t_{a,j+1}},
\end{equation}
where the second inequality is due to $\gamma_{t_{a,j}} \le \alpha_{t_{a,j}}$.

2) $\gamma_{t_{a,j}} > r_{t_{a,j}}^0$. In this case, we have $h_{a}(t_{a,j+1}-1) = h_{a}(t_{a,j}-1)$. We can bound $\tilde{\gamma}_{t_{a,j+1}}$ using $\tilde{\gamma}_{t_{a,j}}$
\begin{equation}
    \tilde{\gamma}_{t_{a,j+1}} \le \tilde{\gamma}_{t_{a,j}} + r^0_{t_{a,j+1}} - \alpha_{t_{a,j}} - \left[\underline{\mu}_L(h_{a}(t_{a,j}-1)) - \Delta_0\right]_{+} \le  r^0_{t_{a,j+1}},
\end{equation}
where the second inequality is due to $\tilde{\gamma}_{t_{a,j}} \le r^0_{t_{a,j}}$ and $\tilde{\gamma}_{t_{a,j}} \le \alpha_{t_{a,j}}$.

Since $\tilde{\gamma}_{t_{a,1}} \le r^0_{t_{a,1}}$, by induction, we have $\tilde{\gamma}_{t_{a,j}} \le r^0_{t_{a,j}}$ for any $a, j$, which concludes the proof.
\end{proof}

\subsection{Proof of \texorpdfstring{\Cref{thm:binary}}{T1}}
\begin{proof}
We use the following two lemmas to prove \Cref{thm:binary}. The proof is similar to those of Lemma 5 and Lemma 6 in \cite{jun2018adversarial}, while as discussed in \Cref{sec:binary}, the algorithm in \cite{jun2018adversarial} relies on \cref{eq:old_alpha} that may not hold due to the binary feedback, and our proof is mainly based on \cref{eq:new_alpha} that can be ensured by the conservative estimations.
Define event $E:= \{\forall i,\forall t > L: |\hat{\mu}^0_{i}(t)-\mu_i| < \beta(N_{i}(t))\}$. With Hoeffding’s inequality, it is easy to prove for $\delta \in (0,1)$, $\mathbb{P}(E) > 1 - \delta$.
\begin{lemma}\label{lemma:time}
Assume event $E$ holds and $\delta \le 1/2$. For any $a \neq L$ and any $t \ge L$, we have
\begin{equation}
    N_a(t) \le \min\{N_L(t), 1+\frac{3}{\Delta_0^2} \log t\}
\end{equation}
\end{lemma}

\begin{proof}
Fix some $t \ge L$. We consider $a_t = a \neq L$ and denote $t' = \max \{ s < t: a_s = a \}$ as the previous round that $a$ was pulled. Since arm $a$ should be pulled at least once by the UCB algorithm, we have $t' \ge 1$, and the attack in round $t'$ would ensure
\begin{equation}\label{eq:binary_lower}
    \hat{\mu}_{a}(t') \le \underline{{\mu}}_{L}(h_{a}(t'))-\Delta_0.
\end{equation}
Since arm $a$ was pulled in round $t$, we know its UCB value must be greater than that of target arm $L$: 
\begin{equation}\label{eq:binary_UCB}
    \hat{\mu}_{a}(t - 1)+\frac{3}{2}\sqrt{\frac{\log (t)}{N_{a}(t - 1)}} \ge \hat{\mu}_{L}(t - 1)+\frac{3}{2}\sqrt{\frac{\log (t)}{N_{L}(t - 1)}}.
\end{equation}
Considering that $t'$ and $t$ are two consecutive rounds when arm $a$ was pulled, we have $\hat{\mu}_{a}(t - 1)=\hat{\mu}_{a}(t')$ and $N_{a}(t - 1)=N_{a}(t')$. Rearranging \cref{eq:binary_UCB}, we get
\begin{equation}
    \frac{3}{2}\sqrt{\frac{\log (t)}{N_{a}(t')}}-\frac{3}{2}\sqrt{\frac{\log (t)}{N_{L}(t-1)}} \ge \hat{\mu}_{L}(t-1)-\hat{\mu}_{a}(t')
    \ge \hat{\mu}_{L}(t-1)-(\underline{{\mu}}_{L}(h_{a}(t'))-\Delta_0) 
    \ge \Delta_0,
\end{equation}
where the second inequality comes from \Cref{eq:binary_lower}. Since $\Delta_0 > 0$, we have
\begin{equation}
    N_a(t)=N_a(t')+1 \le N_L(t-1)=N_L(t).
\end{equation}
Also, since $\frac{3}{2}\sqrt{\frac{\log (t)}{N_{L}(t-1)}}>0$, we have $\frac{3}{2}\sqrt{\frac{\log (t)}{N_{a}(t')}}>\Delta_0$, which implies
\begin{equation}
    N_a(t)=N_a(t')+1 \le 1+\frac{3}{{\Delta_0}^2}\log(t).
\end{equation}
\end{proof}

\begin{lemma}\label{lemma:cost}
Assume event $E$ holds and $\delta \le 1/2$.

1) For any $t \ge  L$, the cumulative attack cost to any fixed arm $a \neq L$ can be bounded as:
\begin{equation}\label{eq:alpha_sum}
\sum_{s\in \tau_{a}(t)} \alpha_s \le N_a(t) \left(\Delta_a + \Delta_0 + 3\beta(N_L(h_a(t)) + \beta(N_a(t))\right) + 1
\end{equation}

2) When t goes to infinity, we have
\begin{equation}
\lim_{t \rightarrow \infty} \frac{\sum_{s\in \tau_{a}(t)} \alpha_s}{\log t} \leq \frac{3}{\Delta_0^2} (\Delta_a + \Delta_0)
\end{equation}
\end{lemma}

\begin{proof}
Fix any arm $a \neq L$. By the definitions of $\gamma_t$ and $\tilde{\gamma}_t$, it follows that:
\begin{align}  
    \sum_{s\in \tau_{a}(t)} \alpha_s & \le \sum_{s\in \tau_{a}(t-1)} \alpha_s+1 \\
    & \le N_{a}(t) \hat{\mu}^0_{a}(t) - N_{a}(t) \left[\underline{\mu}_L(h_{a}(t-1)) - \Delta_0\right]+1 \\
    & \le N_{a}(t) \left[\Delta_a + 3\beta(N_L(h_a(t))) + \beta(N_a(t))\right] + 1,
\end{align}
where the last inequality is due to the decrease of $\beta$.

For the asymptotic result, combining \Cref{lemma:time} with \Cref{{eq:alpha_sum}}, we have
\begin{align}
    \lim_{t \rightarrow \infty} \frac{\sum_{s\in \tau_{a}(t)} \alpha_s}{\log t} & \le \lim_{t \rightarrow \infty} \frac{N_{a}(t)}{\log(t)} \left[\Delta_a + 3\beta(N_L(h_a(t))) + \beta(N_a(t))\right] + \frac{1}{\log(t)} \\
    & \le \lim_{t \rightarrow \infty} (\frac{1}{\log(t)}+\frac{3}{{\Delta_0}^2}) \left[\Delta_a + 3\beta(N_L(h_a(t))) + \beta(N_a(t))\right] + \frac{1}{\log(t)}.
\end{align}
It is easy to check that $\lim_{t \rightarrow \infty}\beta(t)=0$. Hence, to get the asymptotic cost bound in the lemma, we need to prove $\lim_{t \rightarrow \infty}\beta(N_L(h_a(t)))=0$. We find $\lim_{t \rightarrow \infty}h_a(t)=t$ is a sufficient condition for it; in other words, $h_a(t)$ should be always updated when $t$ goes to infinity. To obtain this, we consider two cases of ${\gamma}_t$.

1) If $\underline{\mu}_L(t) - \Delta_0 \ge 0$, we have
\begin{align*}
    \lim_{t \rightarrow \infty}{\gamma}_t
     =& \lim_{t \rightarrow \infty}(N_a(t)\hat{{\mu}}_a^0(t) - \sum_{s\in \tau_{a}(t-1)} \alpha_s - N_a(t)(\hat{\mu}_L(t) -2\beta(N_L(t))-\Delta_0)) \\
     =& \lim_{t \rightarrow \infty}(N_a(t)\hat{{\mu}}_a^0(t) - N_a(t')(\hat{\mu}_L(h_a(t')) -2\beta(N_L(h_a(t'))) -\Delta_0) + r^0_t - \alpha_{t'} \\
    & +  N_a(t')(\hat{\mu}_L(h_a(t')) -2\beta(N_L(h_a(t'))) -\Delta_0) - N_a(t)(\hat{\mu}_L(t) -2\beta(N_L(t))-\Delta_0)) \\
     \le& \lim_{t \rightarrow \infty} (r^0_t - \alpha_{t'} + N_a(t')(\hat{\mu}_L(h_a(t')) -2\beta(N_L(h_a(t'))) -\Delta_0) - N_a(t)(\hat{\mu}_L(t) -2\beta(N_L(t))-\Delta_0)) \\
     \le & r^0_t + \lim_{t \rightarrow \infty} (N_a(t')(\hat{\mu}_L(h_a(t')) - \hat{\mu}_L(t) - 2\beta(N_L(h_a(t')))) - (\hat{\mu}_L(t) - \Delta_0)) \\
     \le & r^0_t.
\end{align*}

The first inequality is due to the attack at round $t'$. The last inequality is due to the confidence radius based on $\beta(N_L(h_a(t')))$ and $\beta$ is decreasing.

2) If $\underline{\mu}_L(t) - \Delta_0 < 0$, we have $\hat{\mu}_L(t) < 2\beta(N_L(t)) + \Delta_0$ and $\hat{\mu}_L(h_a(t')) \le \hat{\mu}_L(t) + 2\beta(N_L(h_a(t')))$.
\begin{align*}
    \lim_{t \rightarrow \infty}{\gamma}_t
    & = r^0_t + \lim_{t \rightarrow \infty}(N_a(t')\hat{\mu}_a^0(t') - \sum_{s\in \tau_{a}(t-1)} \alpha_s) \\
    & = r^0_t + \lim_{t \rightarrow \infty}N_a(t')\hat{\mu}_a(t') \\
    & \le r^0_t + \lim_{t \rightarrow \infty}N_a(t')(\hat{\mu}_L(h_a(t')) - 2\beta(N_L(h_a(t'))) - \Delta_0) \\
    & \le r^0_t + \lim_{t \rightarrow \infty}N_a(t')2\beta(N_L(t)) \\
    & = r^0_t.
\end{align*}
Since in both cases, ${\gamma}_t \le r^0_t$, we have $h_a(t) = t$ when $t$ goes to infinity.
\end{proof}
With \Cref{lemma:time} and \Cref{lemma:cost}, the proof is completed by summing the corresponding upper bounds over all non-target arms $a < L$, 
\end{proof}

\subsection{Proof of \texorpdfstring{\Cref{thm:pbm}}{T2}}\label{sec:proof_T2}
% \textit{Proof sketch}. Whenever $a_{i,t} \notin \bm{a}^*$ is chosen by the player, there must exist $a \in \bm{a}^*$ such that $\bar{\mu}_{a_{i,t}}(t) \ge \bar{\mu}_{a}(t)$. The output $\alpha_{i,t}$ of \Cref{alg:cal_alpha} would ensure $\hat{\mu}_{a_{i,t}(t)} \le \hat{\mu}_{a}(h_{a_{i,t}, a}(t-1)) - 2\beta(N_a(h_{a_{i,t}, a}(t-1))) - \Delta_0$, owing to the conservative estimations in $\gamma_t(a_{i,t},a)$ and $ \tilde{\gamma}_t(a_{i,t},a)$. Combining these two inequalities, we can get $B_{a_{i,t}}(t) - B_{a}(t) \ge \Delta_0$. With a careful calculation on this inequality involved bias-corrected counters, we have $N_{a_{i,t}}(t) \le \left(\frac{1+\epsilon}{2 \kappa^2_{K}\Delta_0^2} \log t\right)$. This result holds for any $a_{i,t} \notin \bm{a}^*$. Thus, we immediately get the bound of $N_L(t)$. The remaining proof for the cost bound will be similar to that of \Cref{thm:binary}. 
\begin{proof}
As for \Cref{alg:UCB}, we first need to prove the post-attack feedback of \Cref{alg:PBM-UCB} is always valid. It is equivalent to showing that the output $\alpha$ of \Cref{alg:cal_alpha} is a valid attack value on the input pre-attack feedback $r^0$, i.e., $\alpha \le r^0$. Similar to the proof of \Cref{lemma:binary}, we consider two cases: when $\gamma_{\max} \le r^0$, $\alpha =  \lceil \gamma_{\max}  \rceil \le r^0$; when $\gamma_{\max} < r^0$, we can use the same inductive proof of \Cref{lemma:binary} to show that $\alpha =  \lceil \tilde{\gamma}_{\max}  \rceil \le r^0$. Thus, the post-attack feedback $r^0 - \alpha$ is always valid. 

Fix some $t \ge L$ such that $a_{i,t} \notin \bm{a}^*$. We denote $t' = \max \{s < t: a_{i,t} \in \bm{a}_s \}$ as the previous round that $a_{i,t}$ was chosen.
With the conservative estimations in $\gamma_t(a_{i,t},a)$ and $ \tilde{\gamma}_t(a_{i,t},a)$, the output $\alpha_{i,t}$ of \Cref{alg:cal_alpha} could ensure 
\begin{equation}\label{eq:PBM-inequ}
    \hat{\mu}_{a_{i,t}}(t) \le \hat{\mu}_{a}(h_{a_{i,t}, a}(t-1)) - 2\beta(N_a(h_{a_{i,t}, a}(t-1))) - \Delta_0,
\end{equation}
for any $a\in\bm{a}^*$.
Since $a_{i,t}$ is chosen by the PBM-UCB algorithm, there must exist $a\in\bm{a}^*$ with its UCB value less than that of $a_{i,t}$:
\begin{equation}\label{eq:PBM-UCB}
    \hat{\mu}_{a}(t-1) + B_{a}(t) \le \hat{\mu}_{a_{i,t}}(t-1) + B_{a_{i,t}}(t).
\end{equation}
Since $t'$ and $t$ are two consecutive rounds when item $a_{i,t}$ was chosen, we have $\hat{\mu}_{a_{i,t}}(t - 1)=\hat{\mu}_{a_{i,t}}(t'), N_{a_{i,t}}(t - 1)=N_{a_{i,t}}(t')$. Rearranging \cref{eq:PBM-UCB}, we have
\begin{equation}
    B_{a_{i,t}}(t) - B_{a}(t) \ge \hat{\mu}_{a}(t-1) - \hat{\mu}_{a_{i,t}}(t') \ge \Delta_0,
\end{equation}
where the last inequality is due to \cref{eq:PBM-inequ}. Since $B_{a_{i,t}}(t) \ge \Delta_0$, with the definition of $B_{a_{i,t}}(t)$ in \cref{eq:PBM-UCB-value}, we have
\begin{equation}
    N_{a_{i,t}}(t) \le \frac{1+\epsilon}{2 \kappa^2_{K}\Delta_0^2} \log t,
\end{equation}
for any $a_{i,t} \notin \bm{a}^*$. Thus, for the target item $L$, 
\begin{equation}
    N_L(T) \ge T - \sum_{l \notin \bm{a}^*} N_{l}(T) \ge T - (L-K) \left(\frac{1+\epsilon}{2 \kappa^2_{K}\Delta_0^2} \log T\right),
\end{equation}
which guarantees the chosen time of the target item. 

For the cumulative attack cost analysis, we consider any arm $a_{i,t} \notin \bm{a}^*$. Since \cref{eq:PBM-inequ} holds for any $a \in \bm{a}^*$, we take $a = L$ and get
\begin{align}  
    \sum_{s\in \tau_{a_{i,t}}(t)} \alpha_{a_{i,t},s} &\le N_{a_{i,t}}(t) \left[\Delta_{a_{i,t}} + \Delta_0 + 3\beta(N_L(h_{a_{i,t}, L}(t))) + \beta(N_{a_{i,t}}(t))\right]\\
    &\le N_{a_{i,t}}(t) \left[\Delta_{a_{i,t}} + \Delta_0 + 4\beta(N_{a_{i,t}}(h_{a_{i,t}, L}(t)))\right],
\end{align}
where the second line is due to $N_{a_{i,t}}(t) \le N_{L}(t), h_{a_{i,t}, L}(t) \le N_{a_{i,t}}(t)$, and $\beta$ is decreasing. Based on this inequality, we can follow the same steps in the proof of \Cref{lemma:cost} to derive the upper bound of the cumulative attack cost for \Cref{alg:PBM-UCB}.
\end{proof}

\subsection{Proof of \texorpdfstring{\Cref{thm:cascade}}{T3}}
% \textit{Proof sketch}. As mentioned above, we use a new regret-based analysis. The intuition is that the regret caused by any suboptimal item is the product of its reward gap with respect to the optimal list and its expected number of recommendations. If we know the regret upper bound and the reward gap lower bound, we can derive the upper bound of the expected number of recommendations. To do so, we first show that the post-attack problem can be viewed as a problem with known reward gap lower bound. This can be obtained by $\hat{\mu}_{a_{i,t}(t)} \le \hat{\mu}_{L}(h_{L_{i,t}, a}(t-1)) - 2\beta(N_a(h_{a_{i,t}, L}(t-1))) - \Delta_0$ for any $a_{i,t} \notin \bm{a}^*$, which indicates the post-attack expected reward gap between $a_{i,t}$ and $L$ is always larger than $\Delta_0$. Then, the lower bound of the reward gap of $a_{i,t}$ with respect to the optimal list is $p^* \Delta_0$. Based on the regret upper bound $12\log T/ \Delta_0 $ given in \cite{kveton2015cascading}, we can get the upper bound of the expected number of recommendations is $\frac{12}{p^{*}\Delta_0^2} \log T$. For the remaining cost analysis, since the attack cost only depends on the observed items, we can still follow the proof of \Cref{thm:pbm}.
\begin{proof}
Since \Cref{alg:CascadeUCB} also calls \Cref{alg:cal_alpha}, we can use the same proof in \Cref{sec:proof_T2} to show its post-attack feedback is always valid. Also, the output $\alpha_{i,t}$ of \Cref{alg:cal_alpha} would still ensure \cref{eq:PBM-inequ} for any $a_{i,t} \notin \bm{a}^*, a \in \bm{a}^*$. Based on this, we can consider the post-attack problem as a cascading bandit problem with known expected reward gaps, where the expected reward of any item $l \notin \bm{a}^*$ is less than that of the target item $L$ by at least $\Delta_0$.
Since $L$ is the worst item in $\bm{a}^*$, with the regret analysis of CascadeUCB in \cite{kveton2015cascading}, we know the regret caused by any list containing $l \notin \bm{a}^*$ is bounded by
\begin{equation}\label{eq:R_upper}
    R_{l}(T) \le \frac{12}{\Delta_0} \log T.
\end{equation}
We also know the lower bound of the one-round instantaneous regret caused by any list containing $l$ is $p^* \Delta_0$, where $p^*:=\prod_{i=1}^{K-1} \mu_i$. We can then write the lower bound of $R_{l}(T)$ as
\begin{equation}\label{eq:R_lower}
    R_{l}(T) \ge p^{*}\Delta_0 \cdot \mathbb{E}[N_l(T)].
\end{equation}
Combinining \cref{eq:R_upper} and \cref{eq:R_lower}, we obtain the upper bound of the expected number of recommendations of item $l$:
\begin{equation}
    \mathbb{E}[N_l(T)] \le \frac{12}{p^{*}\Delta_0^2} \log T.
\end{equation}
Thus, for the target item $L$, 
\begin{equation}
    \mathbb{E}[N_L(T)] \ge T - \sum_{l \notin \bm{a}^*} \mathbb{E}[N_l(T)] \ge T - (L-K) \left(\frac{12}{p^{*}\Delta_0^2} \log T\right).
\end{equation}
For the remaining cost analysis, since the attack cost only depends on the observed items, we can still follow the proof of \Cref{thm:pbm}: we need to change the confidence radius $B_a(t)$ in \cref{eq:PBM-UCB-value} to the typical form $\frac{3}{2}\sqrt{\frac{\log t}{N_a(t-1)}}$, then derive the corresponding cost upper bound.
\end{proof}

\subsection{Proof of \texorpdfstring{\Cref{thm:general}}{T4}}
\begin{proof}
Since \Cref{alg:general} only attacks when $r^0_{i,t} = 1$, it is easy to verify its post-attack feedback is valid. The key step for proving its successful attack is to show that $p_{i,t}$ is an upper bound of $\Delta_{a_{i,t}} / \mu_{a_{i,t}}$: if it is true, by the probabilistic attack with probability $p_{i,t}$, \Cref{alg:general} makes a post-attack bandit problem where the expected reward of any item $l \notin \bm{a}^*$ is worse than that of target item $L$, which guarantees the number of target arm pulls is $T - O(\log T)$. Actually, we have 
\begin{equation}
    \frac{\Delta_{a_{i,t}}}{\mu_{a_{i,t}}} = \frac{\mu_{a_{i,t}} - \mu_{L}}{\mu_{a_{i,t}}} = 1 - \frac{\mu_{L}}{\mu_{a_{i,t}}},
\end{equation}
which increases as $\mu_{a_{i,t}}$ increases. Replacing $\mu_{a_{i,t}}$ with $\hat{\mu}^0_{a_{i,t}}(t) + \beta(N_{a_{i,t}}(t))$,
\begin{equation}
    \frac{\Delta_{a_{i,t}}}{\mu_{a_{i,t}}} \le  \frac{\hat{\mu}^0_{a_{i,t}}(t) + \beta(N_{a_{i,t}}(t)) - {\mu}_{L}}{\hat{\mu}^0_{a_{i,t}}(t) + \beta(N_{a_{i,t}}(t))} \le \frac{\left[\hat{\mu}^0_{a_{i,t}}(t) + \beta(N_{a_{i,t}}(t)) - \hat{\mu}^0_{L}(t) + \beta(N_{L}(t))\right]_{+}}{\hat{\mu}^0_{a_{i,t}}(t) + \beta(N_{a_{i,t}}(t))},
\end{equation}
where $\hat{\mu}^0_{a_{i,t}}(t) + \beta(N_{a_{i,t}}(t)) \ge \mu_{a_{i,t}}$ and $\hat{\mu}^0_{L}(t) + \beta(N_{L}(t)) \ge \mu_{L}$ under event $E$. Since the last term is the definition of $p_{i,t}$, we have proved it is a high-probability upper bound of $\Delta_{a_{i,t}} / \mu_{a_{i,t}}$.  
For the cost analysis, we have
\begin{align}
    C(T) &\le \sum_{i,t} \left[\hat{\mu}^0_{a_{i,t}}(t) + \beta(N_{a_{i,t}}(t)) - \hat{\mu}^0_{L}(t) + \beta(N_{L}(t))\right]_{+} \mathbb{I}\{a_{i,t} \notin \bm{a}^*\}\\
    &\le \sum_{i,t} \left[\mu_{a_{i,t}} + 2\beta(N_{a_{i,t}}(t)) - \mu_{L} + 2\beta(N_{L}(t))\right]_{+} \mathbb{I}\{a_{i,t} \notin \bm{a}^*\}\\
    &\le O \left(\sum_{a < L} \left(\Delta_{a} + 4\beta(1) \right) \log T \right),
\end{align}
which concludes the proof.
\end{proof}

\section{Additional Experiments}
\subsection{Synthetic Data}
We first consider adversarial attacks on $2$-armed stochastic bandits with binary feedback. The rewards of arms $1$ and $2$ follow Bernoulli distributions with means $\mu_1$ and $\mu_2$ ($\mu_1 > \mu_2$). We consider arm $2$ as the target arm. We take $T = 10,000, \delta = 0.1, \Delta_0 = 0.1$.  We set $\mu_1 = 0.85$ and vary the value of $\mu_2$ from $0.03$ to $0.15$.
We compare \Cref{alg:UCB} with a modified algorithm from \cite{jun2018adversarial}, which attacks the non-target arms whenever its calculated $\alpha_t > 0$ and the pre-attack feedback is one. Notice that there is no theoretical guarantee for this modified algorithm. We show the costs and target arm chosen times relative to a trivial baseline, which sets the post-attack feedback of all non-target arms to zero as long as the pre-attack feedback is one.
Figure~\ref{fig:2arm} shows that our algorithm outperforms the modified previous algorithm: it pays fewer costs to achieve more target arm pulls.

\begin{figure}[t]
	\centering
	\begin{subfigure}[b]{0.3\textwidth}
		\centering
		\includegraphics[width=\textwidth]{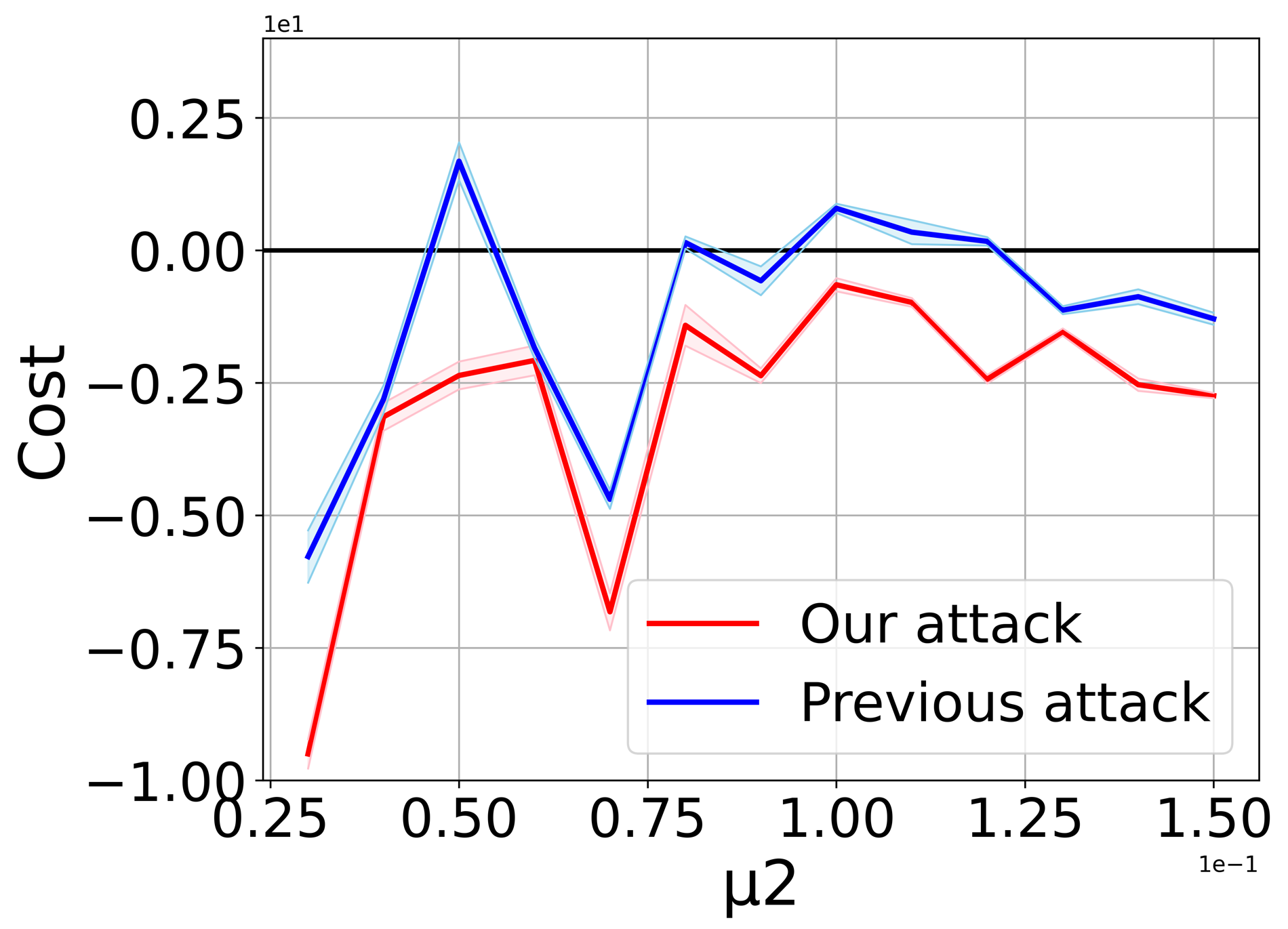}
		\caption{Relative Cost}
		\label{fig:2arm_Cost}
	\end{subfigure}
	\begin{subfigure}[b]{0.3\textwidth}
    	\centering
    	\includegraphics[width=\textwidth]{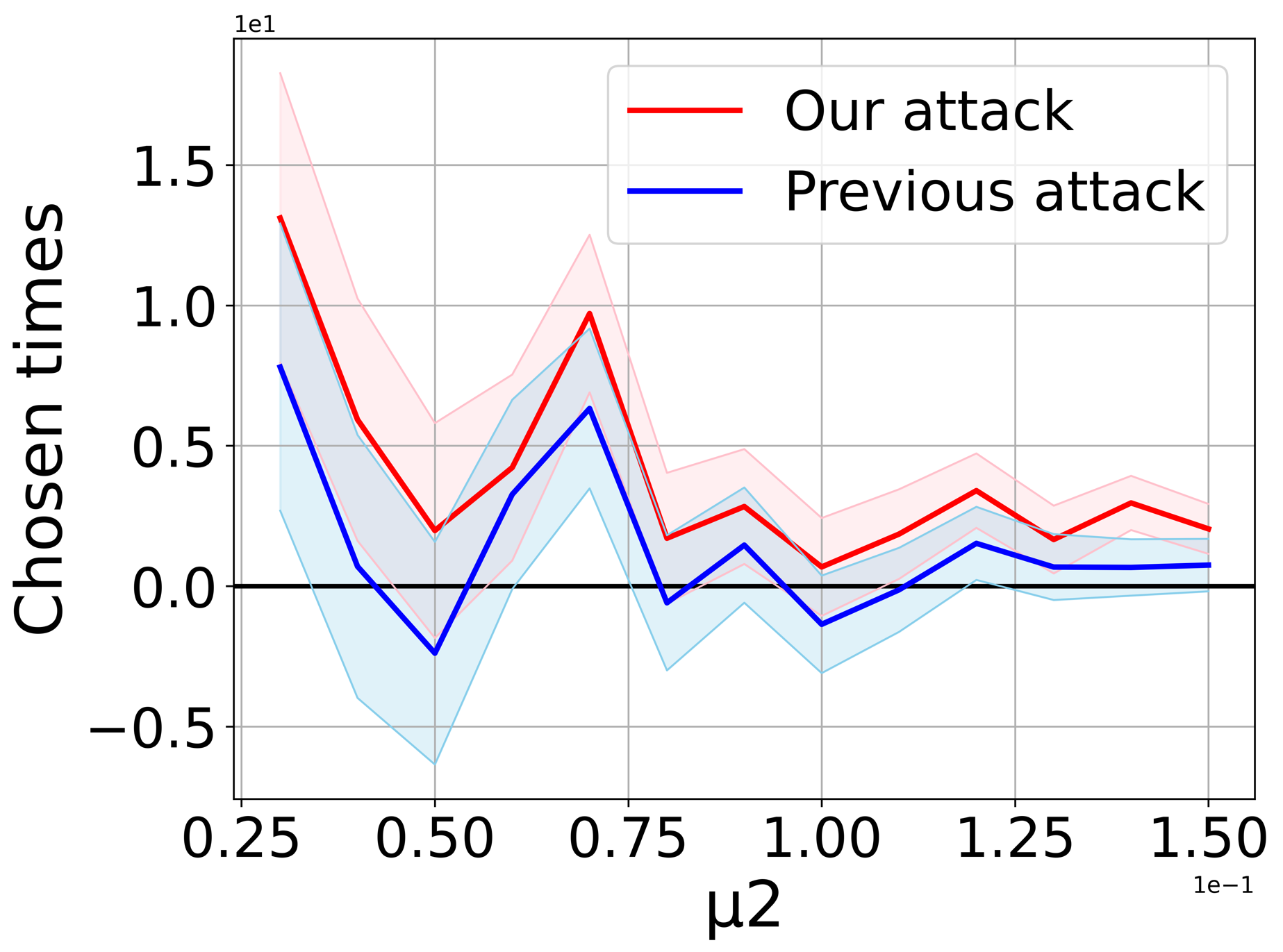}
    	\caption{Relative chosen times}
    	\label{fig:2arm_Times}
	\end{subfigure}
	\caption{Relative costs and target arm chosen times on 2-armed bandits with binary feedback.}
	\label{fig:2arm}
\end{figure}

Next, we study how the algorithm or problem instance parameters affect the performance of \Cref{alg:PBM-UCB}. We use $\epsilon = 0.1$ for the PBM-UCB algorithm. We take $L = 16, K = 8, T = 100,000$; $\{\mu_i\}_{i=1}^{L}$ are sampled from uniform distribution $U(0, 1)$ for Figure~\ref{fig:PBM_delta0}, Figure~\ref{fig:PBM_withAttack}, and sampled from $U(0, x)$ for \Cref{fig:PBM_Sigma}.
Figure~\ref{fig:PBM_delta0} shows that the cumulative cost decreases as $\Delta_0$ increases, which is consistent with our \Cref{thm:pbm}. Figure~\ref{fig:PBM_Sigma} shows that the cumulative cost increases as $x$ increases, which suggests that our algorithm pays more costs when $\Delta_a$ is large.
Figure~\ref{fig:PBM_withAttack} shows that our attack algorithm indeed misguides the PBM-UCB to choose the target arm $T-o(T)$ times.

\begin{figure}[t]
	\centering
	\begin{subfigure}[b]{0.3\textwidth}
    	\centering
		\includegraphics[width=\textwidth]{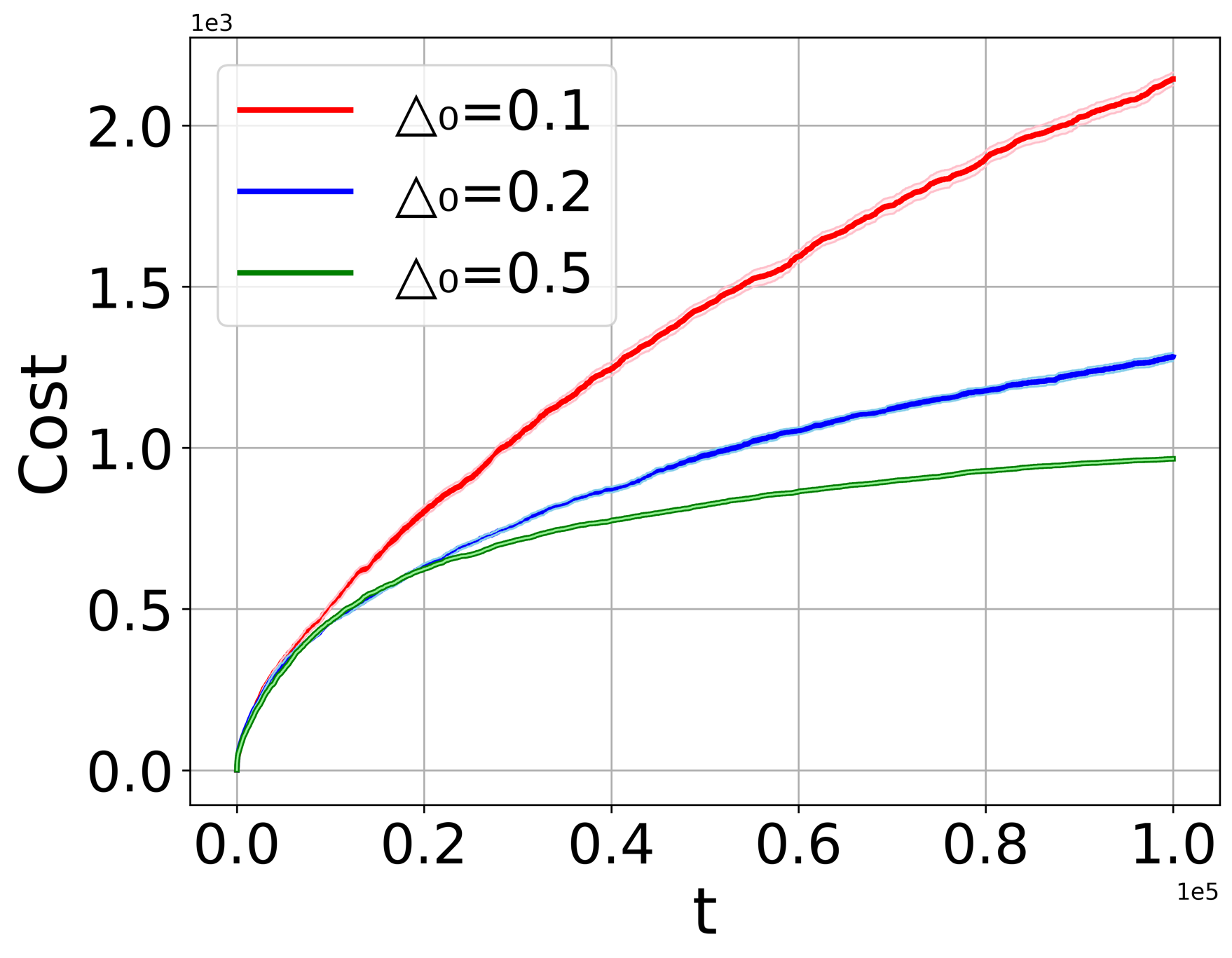}
		\caption{Cost as $\Delta_0$ varies}
		\label{fig:PBM_delta0}
	\end{subfigure}
	\begin{subfigure}[b]{0.3\textwidth}
    	\centering
    	\includegraphics[width=\textwidth]{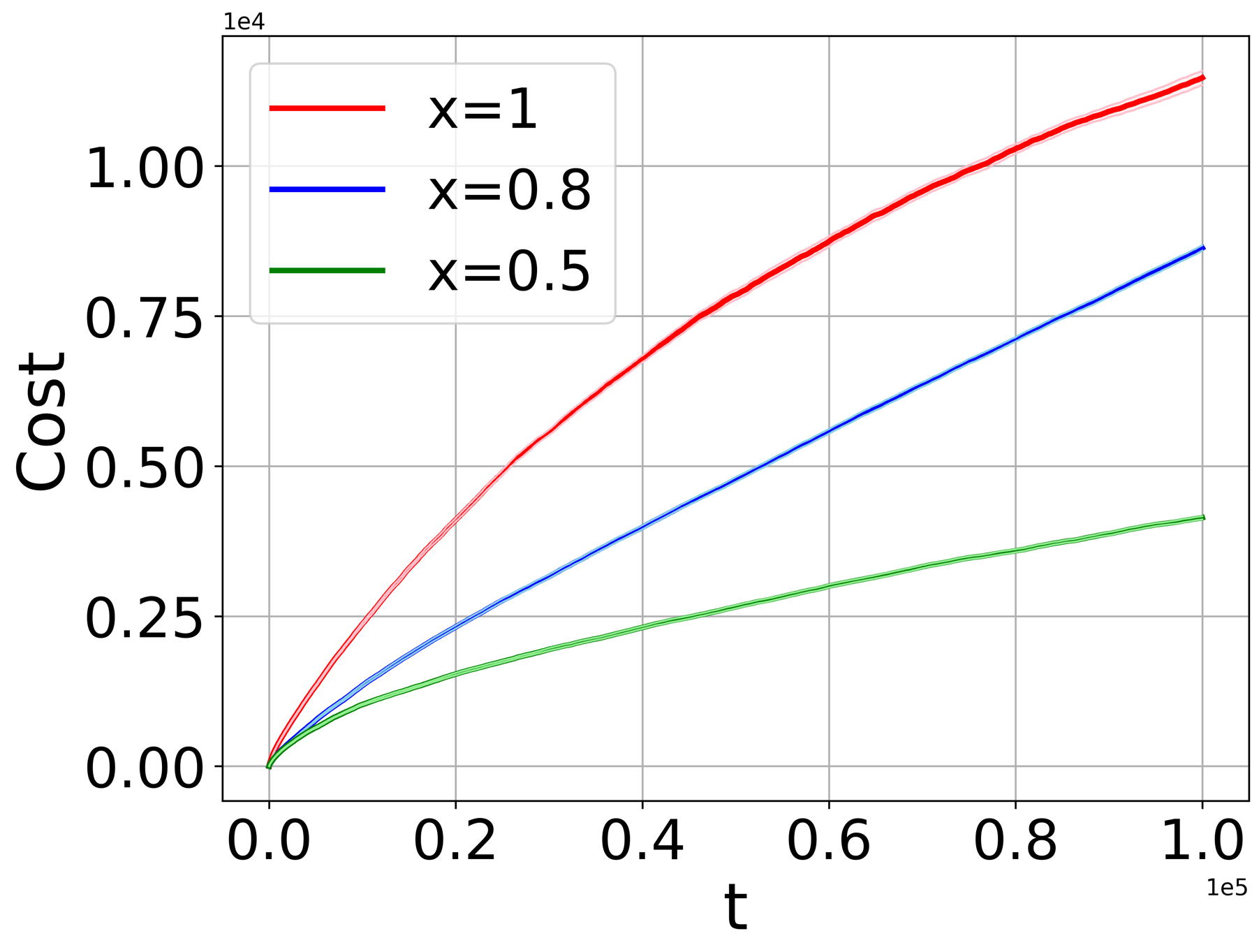}
    	\caption{Cost as $x$ varies}
    	\label{fig:PBM_Sigma}
	\end{subfigure}
        \begin{subfigure}[b]{0.3\textwidth}
    	\centering
    	\includegraphics[width=\textwidth]{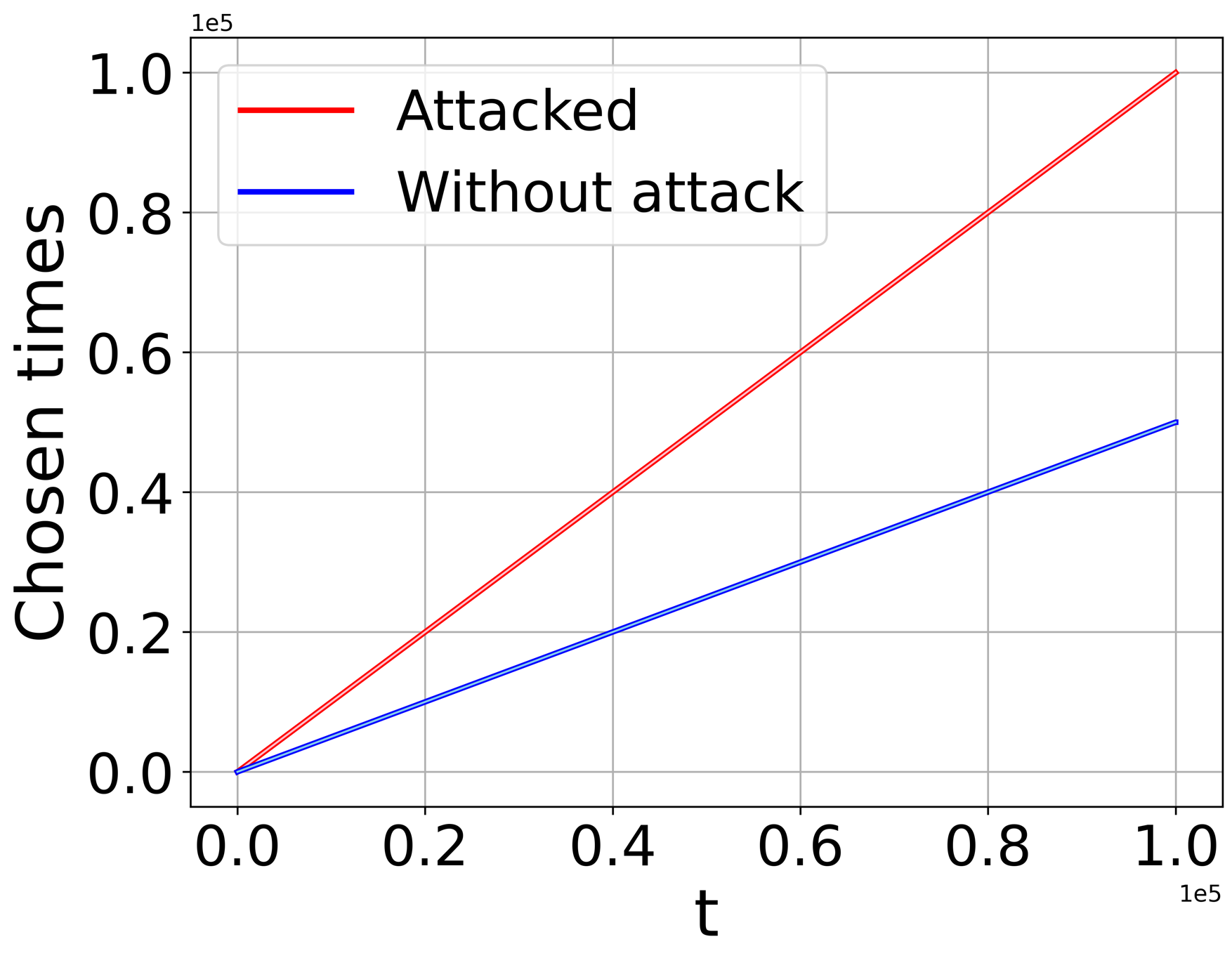}
    	\caption{Chosen times}
    	\label{fig:PBM_withAttack}
	\end{subfigure}
	\caption{Attacks against PBM-UCB with synthetic data.}
	\label{fig:PBM_Parameter}
\end{figure}

We also conduct experiments under the cascade model with synthetic data.
We take $L = 16, K = 8, T = 100,000$; $\{\mu_i\}_{i=1}^{L}$ are sampled from uniform distribution $U(0, 1)$.
We compare \Cref{alg:CascadeUCB} with the same baselines introduced in \Cref{sec:experiment}. Figure~\ref{fig:Cascade} shows that $\texttt{Trivial}_{1}$ algorithm suffers an extremely high cumulative cost, while it cannot misguide the agent to recommend the target arm very often. On the other hand, our algorithm and $\texttt{Trivial}_{K-1}$ algorithm can successfully attack CascadeUCB and perform similarly on the chosen ratio of the target arm. However, our algorithm pays about $30\%$ less cost than $\texttt{Trivial}_{K-1}$.
\begin{figure}[t]
	\centering
	\begin{subfigure}[b]{0.3\textwidth}
		\centering
		\includegraphics[width=\textwidth]{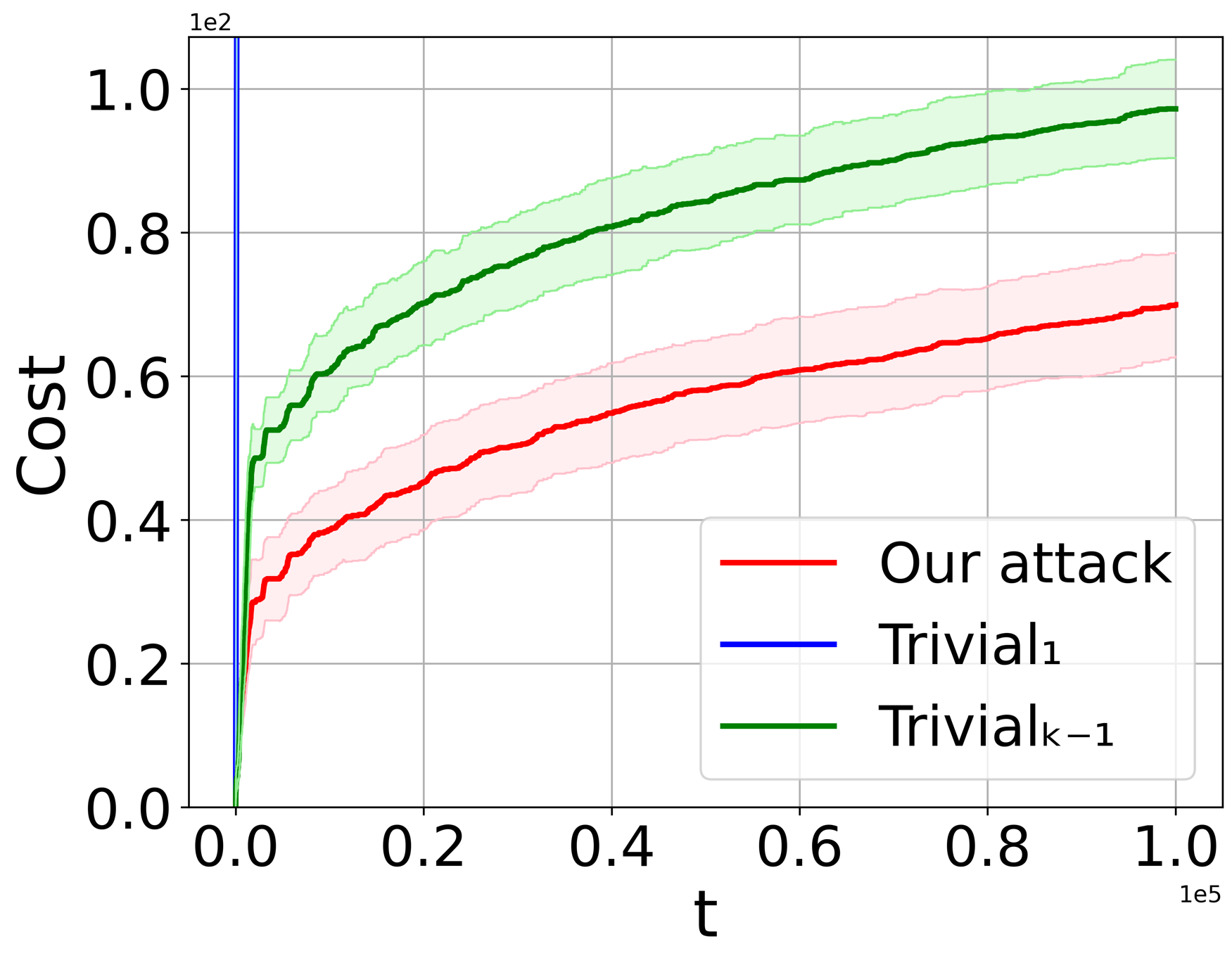}
		\caption{Cost}
		\label{fig:Cascade_Cost}
	\end{subfigure}
	\begin{subfigure}[b]{0.3\textwidth}
    	\centering
    	\includegraphics[width=\textwidth]{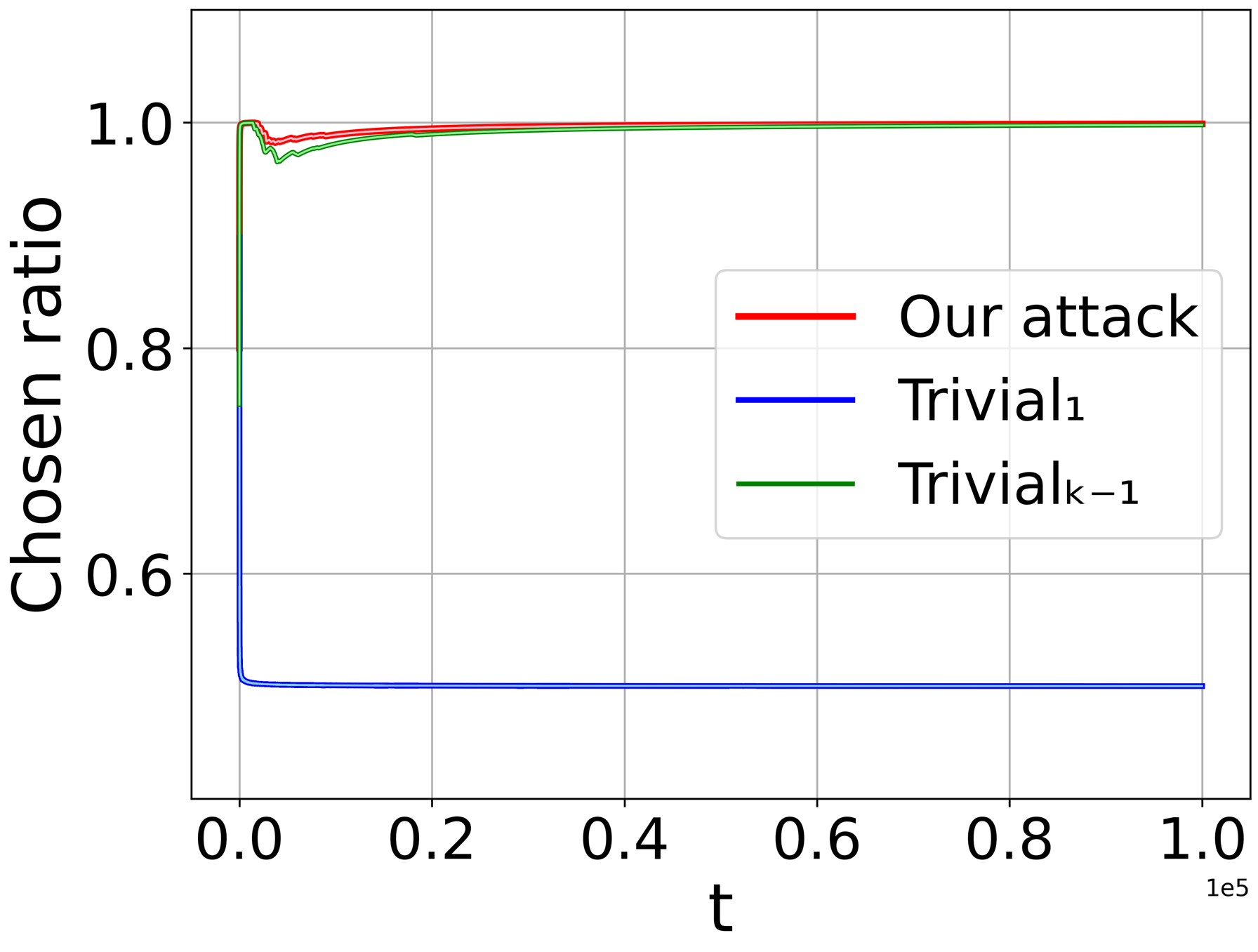}
    	\caption{Chosen times}
    	\label{fig:Cascade_Ratio}
	\end{subfigure}
	\caption{Attacks against CascadeUCB with synthetic data.}
	\label{fig:Cascade}
\end{figure}

\subsection{Real Data}
As discussed in \Cref{sec:experiment}, we have shown results under the position-based model with real data. We now show the experimental results with real data (MovieLens 20M dataset) under the cascade model. We take $L = 100, K = 10, T = 100,000$; $\{\mu_i\}_{i=1}^{L}$ are extracted according to on~\cite{li2019online}. Similarly, we compare the chosen ratio and cumulative cost of our algorithm with $\texttt{Trivial}_{1}$ and $\texttt{Trivial}_{K-1}$.
Figure~\ref{fig:Cascade_realdata} shows that $\texttt{Trivial}_{1}$ cannot successfully attack CascadeUCB as the chosen ratio of the target arm is very low and it suffers a linear cost. \Cref{alg:CascadeUCB} and $\texttt{Trivial}_{K-1}$ have similar performance on the chosen ratio of the target arm. However, our algorithm dramatically decreases the cost by more than $50\%$, which indicates that it is more effective than $\texttt{Trivial}_{K-1}$.
\begin{figure}[t]
	\centering
	\begin{subfigure}[b]{0.3\textwidth}
		\centering
		\includegraphics[width=\textwidth]{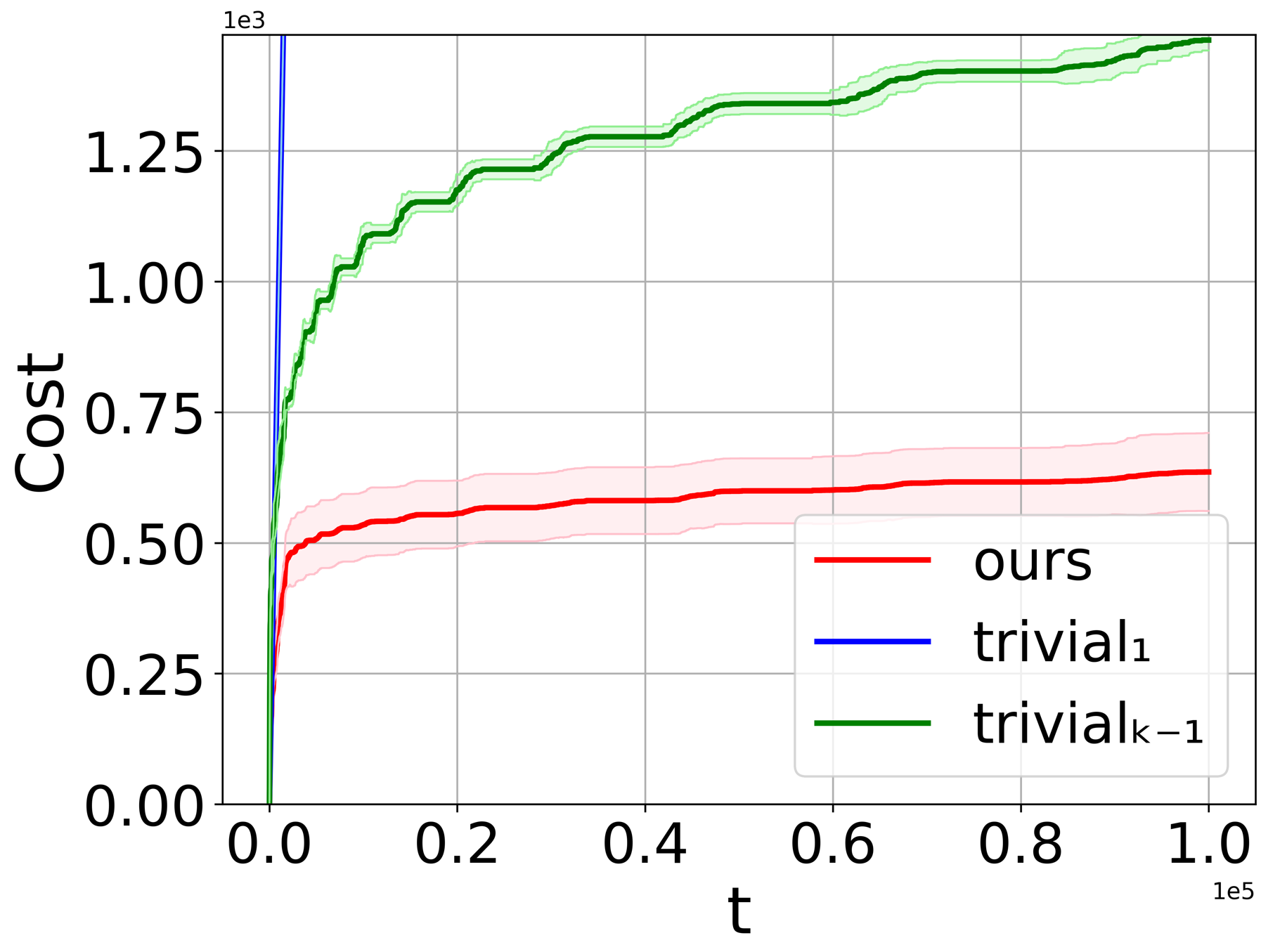}
		\caption{Cost}
		\label{fig:Cascade_Cost_real}
	\end{subfigure}
	\begin{subfigure}[b]{0.3\textwidth}
    	\centering
    	\includegraphics[width=\textwidth]{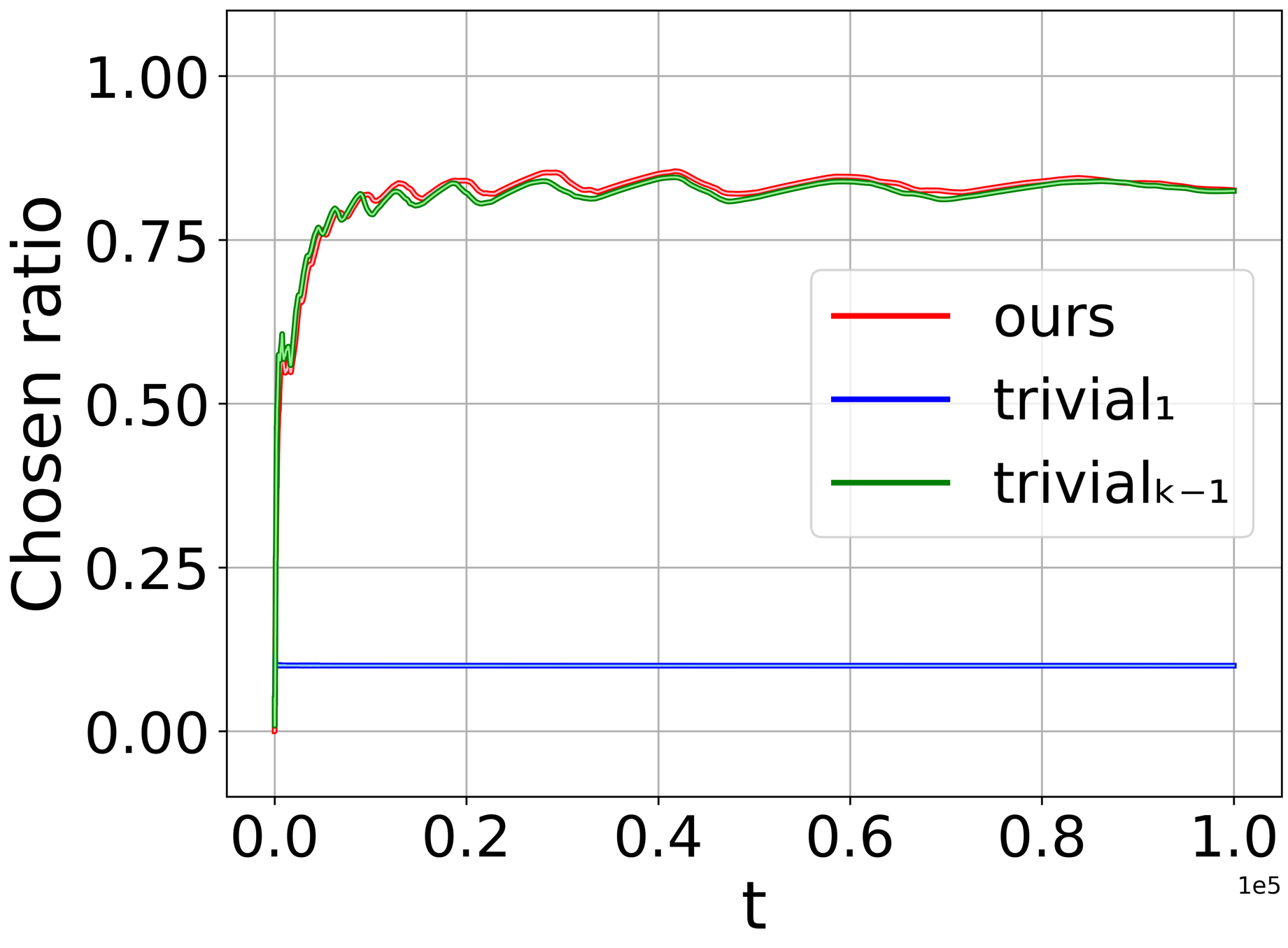}
    	\caption{Chosen times}
    	\label{fig:Cascade_Ratio_real}
	\end{subfigure}
	\caption{Attacks against CascadeUCB with real data.}
	\label{fig:Cascade_realdata}
\end{figure}